\documentclass[preprint,12pt,authoryear]{elsarticle}

\usepackage{amssymb}

\usepackage{microtype}

\usepackage{amsmath}
\usepackage{amssymb}
\usepackage{makecell}
\usepackage{multirow}
\usepackage{array}
\usepackage{ragged2e}
\usepackage[english]{babel}
\usepackage{amsthm}
\usepackage{stmaryrd}
\usepackage[T1]{fontenc}
\usepackage{graphicx}
\usepackage[utf8]{inputenc}
\usepackage{qtree}
\usepackage{ragged2e}
\usepackage{algorithm2e}
\usepackage{kpfonts}
\usepackage{amsfonts}
\usepackage{comment}
\usepackage{longtable}
\usepackage{algpseudocode}
\usepackage{multibib}
\usepackage{anyfontsize}
\usepackage{underscore}
\usepackage{here}
\usepackage[unicode,psdextra]{hyperref}
\usepackage{cleveref}

\newcounter{claimctr}
\newenvironment{case}[1]{\par\refstepcounter{claimctr}\noindent\underline{Case \arabic{claimctr}}.\space#1}{}

\newtheorem{definition}{Definition}

\newtheorem*{note}{Note}
\newtheorem{theorem}{Theorem}
\newtheorem*{theorem*}{Theorem}
\newtheorem{lemma}{Lemma}

\newtheorem{observation}{Observation}
\newtheoremstyle{instruction}
{3\topsep}
{3\topsep}
{}
{}
{\bfseries\itshape}
{:}
{\newline}
{}
\theoremstyle{instruction}
\newtheorem{instruction}{Instruction}

\journal{Theoretical Computer Science}

\begin{document}

\begin{frontmatter}



\title{Minimalist Grammar: Construction without Overgeneration}

\author[inst1]{Isidor Konrad Maier}

\affiliation[inst1]{organization={BTU Cottbus-Senftenberg, Chair of Communication Engineering},
            addressline={Siemens-Halske-Ring 14},
            city={Cottbus},
            postcode={DE-03046},
            }

\author[inst1]{Johannes Kuhn}
\author[inst2]{Jesse Beisegel}
\author[inst1]{Markus Huber-Liebl}
\author[inst1]{Matthias Wolff}

\affiliation[inst2]{organization={BTU Cottbus-Senftenberg, Department of Engineering Mathematics and Numerics of Optimization},
            addressline={Platz der Deutschen Einheit 1},
            city={Cottbus},
            postcode={DE-03046},
            }

\begin{abstract}
  In this paper we give instructions on how to write a minimalist grammar (MG). In order to present the instructions as an algorithm, we use a variant of context free grammars (CFG) as an input format. We can exclude overgeneration, if the CFG has no recursion, i.e. no non-terminal can (indirectly) derive to a right-hand side containing itself.
  The constructed MGs utilize licensors/-ees as a special way of exception handling. A CFG format for a derivation $A\_eats\_B\mapsto^* peter\_eats\_apples$, where $A$ and $B$ generate noun phrases, normally leads to overgeneration, e.\,g., $i\_eats\_apples$. In order to avoid overgeneration, a CFG would need many non-terminal symbols and rules, that mainly produce the same word, just to handle exceptions. In our MGs however, we can summarize CFG rules that produce the same word in one item and handle exceptions by a proper distribution of licensees/-ors.
  The difficulty with this technique is that in most generations the majority of licensees/-ors is not needed, but still has to be triggered somehow. 
  We solve this problem with $\epsilon$-items called \emph{adapters}.
\end{abstract}

\begin{graphicalabstract}
\includegraphics[trim = {0cm 0cm 0cm 0cm}, clip, scale = 0.7]{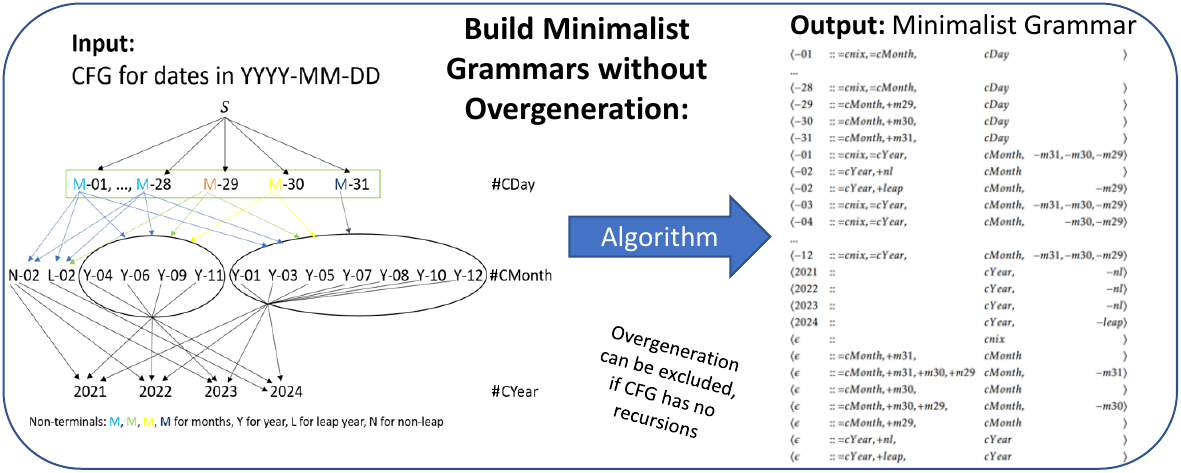}
\end{graphicalabstract}

\begin{highlights}
\item An algorithm is presented to transform a context-free grammar into a Minimalist Grammar.
\item Context-free grammars without recursion can be transformed into a Minimalist Grammar without overgeneration
\item Minimalist Grammars can work with a smaller lexicon than CFGs as they can unite several rules generating the same word to one item. The multiple functionality of the word for several rules is reimplemented with "adapter items".
\item The Python implemented algorithm can be found in \cite{implementation}
\end{highlights}


\end{frontmatter}


\section{State of the Art}
Minimalist grammars (MG) were developed by~\cite{stabler97} to mathematically codify Chomsky’s Minimalist Program (\citeyear{Chomsky1995}).
Beim~\cite{RLMNG} present ideas for constructing small MGs of numeral words.

We have extended these ideas to handcraft larger grammars of numerals and expressions for date and time of day in different languages, see \cite{maier2022minimalist} and \cite{maier-EtAl:2022:EURALI}. In the handcrafted grammars we 
evaded to use several lexical items with the same non-empty exponents. Instead, each non-empty morpheme should not have more than one lexical item. This way, we kept the lexicon small and made lexical parsing easier. Based on the experiences gained from handcrafting these handy grammars, we have developed a general algorithm.

Lately, we got the idea of using a format similar to context-free grammars as the input to transform into an MG.

In \cite{MaierEUNICE}, we have presented an algorithm to decompose numeral words into context-free derivation trees.
It can deliver us a lot of input material to build MGs out of.

About the relation between CFGs and MGs it is well-known that MGs without remnant move are context free, see \cite{kobele2010}.

\cite{amblard2011minimalist} showed that constructions in MGs can be represented as trees and gave further insights into the complexity of MGs.

Regarding general rules of MG modelling, in~\cite{graf2017} there are ideas on how constraints can be represented by features in minimalist syntax. In~\cite{kobele2021} decompositions are modelled with respect to a generalized version of MGs. Some of the most significant present attempts for MG construction can be found in~\cite{indurkhya2019} and~\cite{ermolaeva2020induction}. While Indurkhya's implementation allows overgeneration, Ermolaeva's ideas can overcome this issue by using several lexical items for one word. In~\cite{englischimplementing} another program is offered for computational implementation of Minimalist Syntax.

\cite{stabler-2011-top} showed how to turn an MG into a multiple context free grammar based on theory of \cite{meryMCFG2006}, \cite{Michaelis2001OnFP} and \cite{harkema}. 

We intend to avoid overgeneration without duplication of items. However, it is argued that this task could be moved from the lexica over to the interfaces, see \cite{graf2017,graf2012movement,kobele2011minimalist,kobele2014meeting}.

For the development of minimalist grammars there is a software environment in LKB-style (Linguistic Knowledge Builder), see~\cite{herring2016grammar}. For further information on LKB see~\cite{copestake2002implementing}.

Regarding the question, how small one should sensibly deconstruct morphemes for an MG model, the payoff is analyzed in~\cite{ermolaeva-2021-deconstructing}.

\section{Prerequisites}
In this section we give a deeper introduction into the required terms by providing definitions and some background.

\subsection{Minimalist Grammars}\label{stablermgs}
Minimalist grammars (MGs) are especially well-suited for modelling natural human language, see \cite{Torr2019a,Stabler2013,Fowlie2017,Versley2016,Stanojevic2018}. Within the Chomsky hierarchy their languages are strictly located between the context-free and context-sensitive languages (cf.~\cite{JaegerRogers2012F}, especially Figure 6). Minimalist languages are equivalent to multiple context-free languages and linear context-free rewriting languages (cf.~\cite{Stabler2011AGBT}) as well as mildly context-sensitive languages (cf.~\cite{JaegerRogers2012F}). As those, they establish an infinite hierarchy strictly between context-free and context-sensitive languages (cf.~\cite{Vijay-ShankerWeir1994})

A minimalist grammar consists of a lexicon of items, which store an exponent and syntactic features, and the structure-building functions ``merge'' and ``move''.
The structure-building rules are written in the format
\begin{align*}
    \text{Rule} \frac{Input1\ (Input2)}{Output}.
\end{align*}
We use a notation originating from \cite{STABLER2003345} and write it like in Section II.B of beim~\cite{RLMNG}. There, exponent and syntactic part are given inside a bracket, separated by either `$::$' or `$:$'. $\langle \text{exponent}::\text{syntax}\rangle$ means that the item  is an original item from the lexicon and $\langle \text{exponent}:\text{syntax}\rangle$ means that the item is a derived item. The separator `$\cdot$' is used as a variable placeholder for either `$::$' or `$:$'. The structure-building functions can also be non-final, meaning that they do not (finally) merge the input items but just connect them to a chain.

In the following notation of the structure-building rules, the ${\bf q}$s denote possible parts of a chain, that may be empty. The ${\bf t}$s denote non-empty lists of features.

Selectors like ${=}f$ trigger a merge with an item of category $f$ and licensors like $+f$ trigger a move with a non-finally merged item that holds a licensee $-f$.

\begin{equation*}
\begin{aligned}
&\text{merge-1}
  &&\dfrac{
    \langle e_1 :: {=}f {\bf t}\rangle \quad
    \langle e_2 \cdot\ f\rangle, {\bf q}
  }{
    \langle e_1e_2 : {\bf t}\rangle, {\bf q}
  }&&\text{(final)}\\
&\text{merge-2}
  &&\dfrac{
    \langle e_1 : {=}f {\bf t}\rangle, {\bf q_1} \quad
    \langle e_2 \cdot\ f\rangle, {\bf q_2}
  }{
    \langle e_2e_1 : {\bf t}\rangle, {\bf q_1, q_2}
  }&&\text{(final)}\\
&\text{merge-3}
  &&\dfrac{
    \langle e_1\cdot\ {=}f {\bf t_1}\rangle, {\bf q_1} \quad
    \langle e_2 \cdot\ f {\bf t_2}\rangle, {\bf q_2}
  }{
    \langle e_1 : t_1\rangle, {\bf q_1},
    \langle e_2: t_2\rangle, {\bf q_2}
  }&&\text{(non-final)}\\
&\text{move-1}
  &&\dfrac{
    \langle e_1 : +f {\bf t}\rangle, {\bf q_1},
    \langle e_2 : -f \rangle, {\bf q_2}
  }{
    \langle e_2e_1 : {\bf t}\rangle, {\bf q_1, q_2}
  }&&\text{(final)}\\
&\text{move-2}
  &&\dfrac{
    \langle e_1 : +f {\bf t_1}\rangle, {\bf q_1},
    \langle e_2 : -f {\bf t_2}\rangle, {\bf q_2}
  }{
    \langle e_1 : {\bf t_1}\rangle, {\bf q_1},
    \langle e_2 : {\bf t_2}\rangle, {\bf q_2}
  }&&\text{(non-final)}
\end{aligned}
\end{equation*}
The left items of the inputs are the active items that ``trigger'' the operation, while the right items are the passive. We call them merger/mover and mergee/movee respectively and we use wordings like: \\
``Merger $F$ merges (with) a mergee $X$'' /
``Mover $F$ moves a movee $X$'' \\
``Mergee $X$ is/gets merged by a merger $F$'' /
``Movee $X$ is/gets moved by a mover $F$''\\
At times we simply identify mergers/movers\slash movees with the symbols selector/licensor/licensee.

The five operation rules themselves are no alternatives to each other, but it is entirely determined by (the feature lists of) the input which operation is triggered, see \cite{amblard2011minimalist}. This makes it possible to simply present structures as graphs.

We may refer to MG items simply by their exponent, if the exponent unambiguously belongs to one specific item in a MG in context.

\subsection{Define the input: Context-Free Grammars with Categories}
As we intend to present an algorithm that produces MGs, we need to define an input. We chose ``CFGs with categories", as it is similar to a well-known format that also matches our intuition of the needed input well.

Our definition is based on the underlying definition of CFGs in \cite{HopcroftUllman1979IATLC}.
We suppose that CFGs with categories offer two advantages in that
\begin{itemize}
    \item they can be generated well with the numeral decomposer presented in \cite{MaierEUNICE} and
    \item we regard them as a rather intuitive structure of slot-filling.
\end{itemize}

We present the needed definitions and give some context to each.

\begin{definition}[Context-Free Grammar with Categories]
A context-free grammar with categories is a context-free grammar in which each derivation rule $A\mapsto \alpha$ holds a category $\mathcal{C}$ as another parameter. We will write the rules in the shape
\begin{align*}
    A\mapsto \alpha ~\# \mathcal{C}
\end{align*}
$A$ is called the (producing) non-terminal (NT) or producer or input of the rule, $\alpha$ the word or output of the rule and $\mathcal{C}$ the category of the rule. $\alpha$ is a string out of terminals and non-terminals.
Sometimes we may omit the category, if it is needless to mention in the present context.

When writing out the word of a rule, we consistently use lower-case letters for terminals and upper-case for non-terminals. Lower-case Hellenic letters are used for unspecified words consisting of terminals and non-terminals. We may use lower-case Latin letters, if the word contains only terminals. The empty string is denoted by $\epsilon$.
\end{definition}

In the eventually produced MG, each rule will be represented by an item storing
\begin{itemize}
    \item the rule's word,
    \item the rule's category as the most important syntactic feature
    \item and potentially some more syntactic features.
\end{itemize}

When intending to convert a CFG to a MG using our instructions, one will need to assign the rules to categories, so the MG item of a CFG rule's word gains the category of the rule. In order to implement the top category $\mathcal{C}Fin$, we need to apply our first instruction.
\begin{instruction}
    Add rule $S_0\mapsto S ~\# \mathcal{C}Fin$ to the CFG and reset the start symbol to $S_0$
\end{instruction}
We recommend using categories that refer to the meaning or grammatical property of the rule's word. For example ($A \mapsto$ eat) may be category $\mathcal{C}Action$ or $\mathcal{C}Verb$ and ($B \mapsto$ apples) may be category $\mathcal{C}Noun$ or $\mathcal{C}NounPlural$ or $\mathcal{C}Food$. Grammar-based distributions will facilitate making the MG's language grammatically correct, while meaning-based distribution will rather facilitate excluding semantically undesired expressions like 'apples eat apples'.

A trivial way to add categories, would be to give each rule $A\mapsto \alpha$ a category $\mathcal{C}A$ that simply repeats the rule's NT. However, in this case the MG format can rarely offer advantages over the CFG, especially it cannot work on a smaller lexicon. A major advantage of the MG format is that it can comprise rules $A\mapsto \alpha$ and $B\mapsto \alpha$ into one item for $\alpha$ if they get the same category. Using NTs as categories would eliminate this advantage.

\begin{definition}[Free/Restricted Non-Terminals]\label{def-free}
Given $A$ is a NT in a CFG with categories, $A$ is restricted if there are rules $B\mapsto \beta ~\# \mathcal{C}$ and $A\mapsto \alpha ~\# \mathcal{C}$ of one category with $B\neq A$ but there is no $A\mapsto \beta ~\# \mathcal{C}$. Otherwise, $A$ is a free NT.
\end{definition}

In the MG format, NTs are represented by selectors that target a specific category. If a NT $A$ is free, then each word of the targeted category can be derived from $A$. So, the spot that $A$ holds as a placeholder has free access for all words of the targeted category. Otherwise, if $A$ would be restricted, then according to the definition there would be a rule $B\mapsto \beta$ that produces a word in a target category that $A$ cannot produce. In this case, the access to the spot that $A$ holds as a placeholder, would need to to be restricted, so that $b$ cannot get there despite belonging to a target category.

\begin{definition}[Recursion Free and Category-Recursion Free]
A CFG with categories is recursion-free, if for any NT $A$ and any series of at least one derivation $A\mapsto \alpha_1\mapsto \alpha_2\mapsto\ldots \alpha_n$, the NT $A$ cannot appear in $\alpha_n$. If additionally $\alpha_n$ cannot contain a NT symbol producing a rule of the same category as $A\mapsto \alpha_1$, then the CFG is category-recursion free.
\end{definition}

The property category-recursion free is significant, because we can prove that MGs built out of category-recursion free CFGs do not overgenerate, e.g. they do not generate expressions apart from those that the input CFG does.
Any recursion-free CFG can be made category-recursion free by favorable distribution of categories. All it needs is a sufficient separation of the categories, where any pair of rules that violates the condition is separated.
\begin{note}
The relation $\succ$ of categories defined as
\begin{equation}\label{catorder}
\begin{aligned}
    \mathcal{C}\ \succ\ \mathcal{C}' :\Leftrightarrow \\ \exists\text{ rules } (A\mapsto \alpha ~\# \mathcal{C}), (A'\mapsto \alpha' ~\# \mathcal{C}'): \\ \exists\text{ derivation }A\mapsto \alpha\mapsto \ldots\mapsto \omega\text{ with }A'\text{ in } \omega
\end{aligned}
\end{equation}
is a strict partial order if the overlying CFG is category-recursion free.
\end{note}

\section{Basic Models for Single CFG rules}
In this section we show how a derivation of one CFG word can be transformed into a derivation of MG items in a vacuum, i.e. we ignore for now that the grammar may have more items that interfere with the construction. A word - or output of a rule - can generally be represented as $s_0X_1s_1X_2\ldots X_ns_n$, where the $s_i$ are strings of terminals and $X_j$ are (free or restricted) non-terminal symbols. We show how the rules $X_j\mapsto x_j$ with a string $x_j$ of terminals for $j=1,\ldots,n$ can be applied on the word.

We explain the modelling in three steps:
\begin{enumerate}
    \item The 3 basic shapes of a word, where the MG model is the simplest
    \item Resolve the issue of a NT producing words in several categories
    \item How to model a word of general shape as a combination of the basic shapes
\end{enumerate}
General notes on the MGs:
\begin{itemize}
    \item Categories whose names end in an apostrophe, as well as categories with an exponent are auxiliary categories. Auxiliary categories are singletons, i.\,e., they only contain one item.
    \item Some of the (input) items in the grammars are written with a `$\cdot$' instead of a `::'. It means that they could already be derived items, since usually a preceding adaption of the feature list is necessary before modelling the derivation from a NT. The adaption procedure is described in Section 4. \\As derived items cannot be listed in a lexicon, the following lists of items should technically be seen as workspaces rather than as grammars.
    \item The workspaces also do not contain any item of category $\mathcal{C}Fin$, which implies that the given trees are not only incomplete at the bottom, but also at the top. They are just coutouts of a derivation tree.
    \item It follows from the structure rules that the given items of each workspace cannot be composed differently than intended. Moreover, the distribution of licensees/-ors makes sure that a word can be derived from a restricted NT if and only if it holds the required licensee.
\end{itemize}

\subsection{The word shapes easiest to model by MGs}\label{basic-wordshapes}
In this subsection we present three different MG workspaces that can generate the expression $peter\_eats\_apples$.

In the trivial case where the expression is a single word without composition, the MG item would be $\langle peter\_eats\_apples :: \mathcal{C}Sentence\rangle$. \\Next we present the easy to model words that actually contain NTs.
\begin{case}
    Right-free: We can model the derivation
    \begin{align*}
        peter\_eats\_F \mapsto^* peter\_eats\_apples
    \end{align*}
    by creating items $\langle peter\_eats\_ :: {=}\mathcal{C}Food, \mathcal{C}Sentence\rangle$ and $\langle apples :: \mathcal{C}Food\rangle$. The selector ${=}\mathcal{C}Food$ can trigger a merge1 with any item of the shape $\langle \ldots \cdot \mathcal{C}Food\rangle$, which would merge the mergee to the right-hand side:
    \begin{align}\label{merge:rf}
    \text{merge-1}&&
  \dfrac{
    \langle peter\_eats\_ :: {=}\mathcal{C}Food, \mathcal{C}Sentence\rangle \quad
    \langle apples :: \mathcal{C}Food\rangle
  }{
    \langle peter\_eats\_apples : \mathcal{C}Sentence\rangle
  }
    \end{align}
    Merge-2 would only be triggered, if the merger already was a derived item. Merge-3 would be triggered, if the mergee had licensors left in its feature list.
\end{case}
\begin{case}
    Two-handed free: We can model the derivation
    \begin{align*}
        P\_eats\_F \mapsto^* peter\_eats\_apples
    \end{align*}
    by creating items $\langle \_eats\_ :: {=}\mathcal{C}Food, {=}\mathcal{C}Person, \mathcal{C}Sentence\rangle$, $\langle apples :: \mathcal{C}Food\rangle$ and $\langle peter :: \mathcal{C}Person\rangle$.
    First, the selector ${=}\mathcal{C}Food$ can trigger a merge1 with any item of the shape $\langle \ldots \cdot \mathcal{C}Food\rangle$, which would merge the mergee to the right-hand side:

    \begin{align}\label{merge:2hf1}
        \text{merge-1}
  &&\dfrac{
    \langle \_eats\_ :: {=}\mathcal{C}Food, {=}\mathcal{C}Person, \mathcal{C}Sentence\rangle \quad
    \langle apples :: \mathcal{C}Food\rangle
  }{
    \langle \_eats\_apples : {=}\mathcal{C}Person, \mathcal{C}Sentence\rangle
  }
    \end{align}
    Then the selector $=\mathcal{C}Person$ can trigger another final merge with any item of the shape $\langle \ldots \cdot \mathcal{C}Person\rangle$. As the merger is a derived item then, a merge-2 is triggered, which merges the mergee to the left-hand side:
    \begin{align}\label{merge:2hf2}
        \text{merge-2}
  &&\dfrac{
    \langle \_eats\_apples : {=}\mathcal{C}Person, \mathcal{C}Sentence\rangle \quad
    \langle peter :: \mathcal{C}Person\rangle
  }{
    \langle peter\_eats\_apples : \mathcal{C}Sentence\rangle
  }
    \end{align}
\end{case}
\begin{case}
    Left-restricted: We can model the derivation
    \begin{align*}
        N\_eats\_apples \mapsto^* peter\_eats\_apples
    \end{align*}
    by creating items $\langle \_eats\_apples :: {=}\mathcal{C}Noun, +person, \mathcal{C}Sentence\rangle$ and $\langle peter :: \mathcal{C}Noun, -person\rangle$.\footnote{The item for $peter$ has been intentionally modelled different here than in the two-handed free case, since we want to model $N$ as a restricted NT.} The selector ${=}\mathcal{C}Noun$ can trigger a merge with any item that has its category $\mathcal{C}Noun$ as its foremost feature. However, if the mergee has no licensor $-person$ that matches the mergers $+person$, then the construction meets a dead end.
    This way the licensee $+person$ indirectly forbids the merge with some items of category $\mathcal{C}Noun$, so unsensible expressions like $stone\_eats\_apples$ can be avoided.
    If $\langle\_eats\_apples :: {=}\mathcal{C}Noun, +person, \mathcal{C}Sentence\rangle$ performs a merge with an item of the shape $\langle \ldots \cdot \mathcal{C}Noun, -person\rangle$, it is a non-final merge-3, as the mergee has licensors left in its feature list:
    {\fontsize{10.0}{12.0}\selectfont
    \begin{align}\label{merge:lr}
        \text{merge-3}
  &&\dfrac{
    \langle\_eats\_apples :: {=}\mathcal{C}Noun, +person, \mathcal{C}Sentence\rangle \quad
    \langle peter :: \mathcal{C}Noun, -person\rangle
  }{
    \langle \_eats\_apples : +person, \mathcal{C}Sentence\rangle,
    \langle peter : -person\rangle
  }
    \end{align}
    }
    Then $+person$ triggers a final move-1, because $-person$ is the only licensor left in the movees feature list. move1 moves the movee to the left-hand side:
    \begin{align}\label{move:lr}
        \text{move-1}
  &&\dfrac{
    \langle \_eats\_apples : +person, \mathcal{C}Sentence\rangle,
    \langle peter : -person\rangle
  }{
    \langle peter\_eats\_apples : \mathcal{C}Sentence\rangle
  }
    \end{align}
\end{case}

\subsection{Non-Terminals with Several Target Categories}
In the previous subsection we presented the most basic MG models of CFG words. However, in each case an NT could only be derived into words of one specific target category.

But what if we have rules $A\mapsto \alpha ~\# \mathcal{C}$ and $A\mapsto \beta ~\# \mathcal{C}'$?
Or generally, if $\mathcal{C}_1,\ldots,\mathcal{C}_n$ is the set of categories, for which a rule $A\mapsto \alpha_i ~\# \mathcal{C}_i$ of the NT $A$ exists?

We resolve the issue, by choosing a main target category $\mathcal{C}_m$. Then, for $i\in \{1,\ldots,n\}\setminus\{m\}$, introduce $A_i$ as a new auxiliary NT and replace each rule $A\mapsto \alpha_i ~\# \mathcal{C}_i$ by two rules $A\mapsto A_i ~\# \mathcal{C}_m$ and $A_i\mapsto \alpha_i ~\# \mathcal{C}_i$.
This way, $A$ has one unique target category $\mathcal{C}_m$ and the new auxiliary NTs $A_i$ also have only one unique target category $\mathcal{C}_i$.

As an example, consider the following snippet of a CFG for numerals:
\begin{align*}
    &\ldots &&\mapsto hundredandA &&~\# \mathcal{C}3\\
    &A &&\mapsto thirty-B &&~\# \mathcal{C}2\\
    &A &&\mapsto two &&~\# \mathcal{C}1
\end{align*}
Here, the numerals are categorized after their number of digits. Since $hundredA$ allows 1-digit numerals as well as 2-digit numerals as a suffix, the NT $A$ targets both category $\mathcal{C}2$ and $\mathcal{C}1$. In order to resolve this issue, $\mathcal{C}2$ can be chosen as the the main target category of $A$. Then the rule $A \mapsto two ~\# \mathcal{C}1$ would have to be replaced by two rules $A\mapsto A_1 ~\# \mathcal{C}2$ and $A_1 \mapsto two ~\# \mathcal{C}1$.

Note that such newly added rules $A\mapsto A_i ~\# \mathcal{C}_m$ generally have an impact on the relation $\prec$ defined in Equation~\ref{catorder}.
$\mathcal{C}_m$ should be chosen "maximal" with respect to $\prec$, i.e. so that $\mathcal{C}_m\nprec\mathcal{C}_i$ for $i=1,\ldots,n$. Then $\prec$ keeps being a strict partial order, if it was before.

In the present example, if we consider that there is another rule \\$B\mapsto~two~\#\mathcal{C}1$, then $\mathcal{C}1$ should not be chosen as the main target category of $A$. In this case, $A\mapsto A_1 ~\# \mathcal{C}1$, $A_1 \mapsto thirty-B ~\# \mathcal{C}2$ and $B \mapsto two ~\# \mathcal{C}1$ would form a category recursion and $\prec$ would no longer be a strict order.

In contrast, $\mathcal{C}2$ would be proper choice for the main target category, since the rule $B \mapsto two ~\# \mathcal{C}1$ with $B$ originating from the word $thirty-B$ of category $\mathcal{C}2$ means that $\mathcal{C}1\prec \mathcal{C}2$.

Is such proper choice always possible? Yes, unless $\prec$ was not strict before anyway. Each strict partial order $\prec$ is a subset of a total order $\prec'$, so the "maximum" can be chosen with respect to some fixed total order $\prec'$.
Here we can continue our instructions:
\newpage
\begin{instruction}
    Normalize each NT's target category:
    \begin{itemize}
        \item[] For each NT $A$:
        \begin{itemize}
            \item[] If $A$ has rules of more than one category:
            \begin{itemize}
                \item[] Choose main category $\mathcal{C}_m$ so that neither of the other categories have a word containing a NT that has a rule of $\mathcal{C}_m$.
                \item[] For all other categories $\mathcal{C}_i$ of $A$:
                \begin{itemize}
                    \item[] introduce a new auxiliary NT $A_i$ and add a rule $A\mapsto A_i ~\# \mathcal{C}_m$
                    \item[] replace all rules $A\mapsto \ldots ~\# \mathcal{C}_i$ by $A_i\mapsto \ldots \#\mathcal{C}_i$
                \end{itemize}
            \end{itemize}
        \end{itemize}
    \end{itemize}
\end{instruction}
When the target category of each NT is normalized, one can assess which NT is free and which is restricted.
\begin{instruction}
    Classify all NTs as free or restricted with respect to Definition~\ref{def-free}.
\end{instruction}
\subsection{MG model of general word shape}
Before the complicated general case we explain two tricks separately in edge cases.

\begin{enumerate}
    \item Right restricted: We have seen in section 3.1 that that the MG items
    \begin{align*}
        \langle peter\_eats\_ :: {=}\mathcal{C}Noun, +food, \mathcal{C}Sentence\rangle\\\text{ and }\langle apples :: \mathcal{C}Noun, -food\rangle
    \end{align*}
    would produce $applespeter\_eats\_$, rather than $peter\_eats\_apples$. So, there is no straight way to model the CFG word $peter\_eats\_N$ with a restricted $N$. Instead, the rule $A\mapsto peter\_eats\_N ~\# \mathcal{C}Noun$ can be decomposed into $A\mapsto peter\_eats\_N' ~\# \mathcal{C}Noun'$ and $N'\mapsto N ~\# \mathcal{C}N$. Then the word $N$ can be interpreted as left restricted (as $N\epsilon$), and $peter\_eats\_N'$ is right free, as the auxiliary category $\mathcal{C}N'$ is new and unique to $N'$.
    In MG format the construction is then:

    \begin{figure}[H]
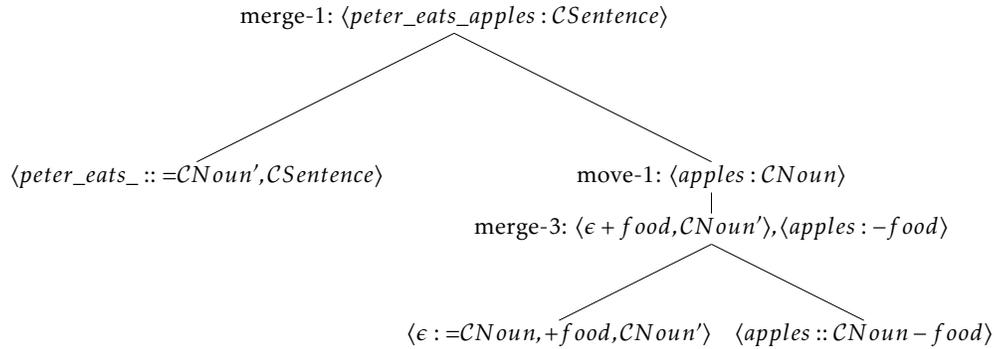

    \centering
    \scalebox{0.75}{
    \Tree [.{merge-1: $\langle peter\_eats\_apples : \mathcal{C}Sentence\rangle$}     [.{$\langle peter\_eats\_ :: {=}\mathcal{C}Noun', \mathcal{C}Sentence\rangle$} ][.{move-1: $\langle apples : \mathcal{C}Noun\rangle$} [.{merge-3: $\langle \epsilon +food,\mathcal{C}Noun'\rangle,\langle apples : -food\rangle$} [.{$\langle \epsilon : {=}\mathcal{C}Noun, +food, \mathcal{C}Noun'\rangle$} ][.{$\langle apples :: \mathcal{C}Noun -food\rangle$} ] ] ] ]
    }\caption{Derivation model for a right restricted NT}
    \end{figure}
    \item Left free: We have seen in Section 3.1 that the MG items
    \begin{align*}
        \langle\_eats\_apples :: {=}\mathcal{C}Person\rangle\\\text{ and }\langle peter :: \mathcal{C}Person\rangle
    \end{align*}
    would produce $\_eats\_applesPeter$, rather than $Peter\_eats\_apples$. As $\langle\_eats\_apples :: {=}\mathcal{C}Person\rangle$ is an original lexical item, ${=}\mathcal{C}Person$ triggers a merge1 that merges $peter$ to the right-hand side. However, we can remodel the CFG word $P\_eats\_apples$ to a two-handed free $P\_eats\_applesO$, where $O$ can perform a neutral derivation with the rule $O\mapsto \epsilon ~\# \mathcal{C}nix$.\footnote{'nix' means 'nothing' in a variant of Deutsch (German)}
    Then the MG construction would work:

    \begin{figure}[H]
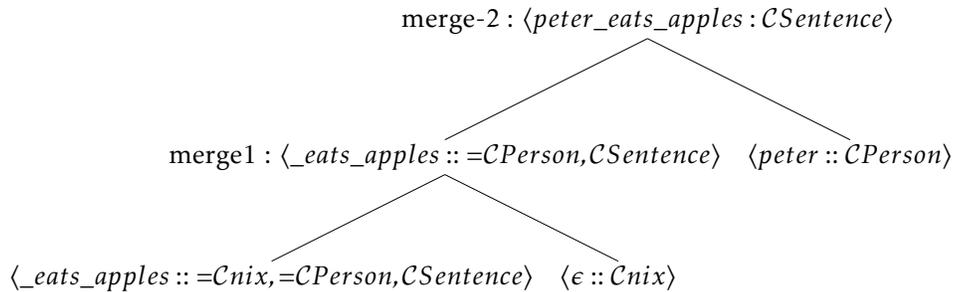

    \centering
    \scalebox{0.85}{
    \Tree [.{merge-2 : $\langle peter\_eats\_apples : \mathcal{C}Sentence\rangle$}     [.{merge1 : $\langle \_eats\_apples :: {=}\mathcal{C}Person, \mathcal{C}Sentence\rangle$} [.{$\langle \_eats\_apples :: {=}\mathcal{C}nix, {=}\mathcal{C}Person, \mathcal{C}Sentence\rangle$} ][.{$\langle \epsilon :: \mathcal{C}nix \rangle$} ] ][.{$\langle peter :: \mathcal{C}Person\rangle$} ] ] }
    \caption{Derivation model for a left free NT}
    \end{figure}
\end{enumerate}
As seen in the edge cases, the way is to introduce auxiliary rules with auxiliary NTs. The auxiliary NTs are always free as they target their own unique auxiliary category. By using these auxiliary rules, a word can be decomposed into several words, so that each of them is either right free or left restricted or 2-handed free.

Now we consider a general rule $A\mapsto s_0X_1s_1X_2\ldots X_ns_n ~\# \mathcal{C}$, where the $s_i$ are strings of terminals and the $X_j$ are (free or restricted) NTs that target the category $\mathcal{C}_j$.
\begin{instruction}
Decompose each rule
    \begin{align*}
        A\mapsto s_0X_1s_1X_2\ldots X_ns_n ~\# \mathcal{C}
    \end{align*}
    into several rules
    \begin{align*}
    &A&\mapsto s_0&A': &&&&&&&&\mathcal{C}\\
    &&&            A'&\mapsto X_1s_1&A'': &&&&&&\mathcal{C}'\\
    &&&&&                            \vdots\\
    &&&&&                            A^{(i)}&\mapsto X_is_i&A^{(i+1)}: &&&&\mathcal{C}^{(i)}\\
    &&&&&&&                                                 \vdots\\
    &&&&&&&                                                 A^{(n)}&&\mapsto X_ns_n: &&\mathcal{C}^{(n)},
    \end{align*}
    under usage of the following exceptions:
    \begin{itemize}
        \item If $s_0$ is $\epsilon$, summarize the first two rules to $A\mapsto X_1s_1A''$.
        \item If $s_n$ is $\epsilon$ and $X_n$ is free, summarize the last two rules  to $A^{(n-1)}\mapsto X_{n-1}s_{n-1}X_n: \mathcal{C}^{(n-1)}$.
        \item For $i=1,\ldots,n-1$, if $X_i$ is restricted, replace the rule $A^{(i)}\mapsto X_is_iA^{(i+1)}: \mathcal{C}^{(i+1)}$ by two rules $A^{(i)}\mapsto X'_is_iA^{(i+1)}: \mathcal{C}^{(i+1)}$ and $X'_i\mapsto X_i: \mathcal{C}'_i$.
        \item If $X_n$ is free and $s_n$ is not $\epsilon$, then replace $A^{(n)}\mapsto X_ns_n$ by two rules $A^{(n)}\mapsto X_ns_nO: \mathcal{C}^{(n)}$ and $O\mapsto \epsilon: \mathcal{C}nix$.
    \end{itemize}
\end{instruction}
The last two exceptions utilize the right restricted and left free edge cases respectively. Funnily enough, these "edge cases" only appear at the edges of word in general, so the wording makes sense in several ways.

After following this instruction, all rules of the CFG are either right free or left restricted or 2-handed free or entirely terminal.
\begin{instruction}
    For each right free word $wordX$ of category $\mathcal{C}$ create a MG item
    \begin{align*}
        \langle word :: {=}\mathcal{C}(X), \mathcal{C}\rangle.
    \end{align*}
    For each 2-handed free word $YwordX$ of category $\mathcal{C}$ create a MG item
    \begin{align*}
        \langle word :: {=}\mathcal{C}(X), {=}\mathcal{C}(Y), \mathcal{C}\rangle.
    \end{align*}
    For each left restricted word $Yword$ of category $\mathcal{C}$ create a MG item
    \begin{align*}
        \langle word :: {=}\mathcal{C}(Y), +a_Y, \mathcal{C}\rangle.
    \end{align*}
    For each terminal word $word$ of category $\mathcal{C}$ create a MG item
    \begin{align*}
        \langle word :: \mathcal{C}\rangle.
    \end{align*}
\end{instruction}
\begin{definition}[Word Item]
    All MG items introduced in Instruction 5 are called word items, as they represent a (part of a) CFG word.
\end{definition}
\begin{instruction}\label{distr-licensee}
    Give to each MG item $\langle word :: \ldots \mathcal{C}\rangle$ - which represents the CFG word $wordX$, $YwordX$, $Yword$ or $word$ in category $\mathcal{C}$ - one licensee $-a_R$ for each restricted NT $R$ that generates its CFG word.
\end{instruction}


\section{From Basic Models to Complete MGs}
In Section 3, we have shown how to MG model a direct derivation of one CFG word. In the presented minimal examples, many obstacles do not appear yet.\\
In a real application however, there are several words and many NTs and each word may be derivable from several NTs. Possibly also from several restricted NTs, in which case by Instruction~\ref{distr-licensee} it would hold several licensees that could interfere with each other.

The problem and the evolution of the solution is explained in the following example:\\
The Deutsch\footnote{Endonym for German} numeral `vier' can be extended to larger numerals like `vierzehn', `vierzig' or `vierundzwanzig'. We start modelling this with the following items:
\begin{align*}
&\langle \text{zig} :: {=}\mathcal{C}1, +zi, \mathcal{C}2\rangle\\
&\langle \text{zehn} :: {=}\mathcal{C}1, +zeh, \mathcal{C}2\rangle\\
&\langle \text{undzwanzig} :: {=}\mathcal{C}1, +un, \mathcal{C}2\rangle\\
&\langle \text{vier} :: \mathcal{C}1, -zi, -zeh, -un\rangle
\end{align*}
This way `vier' has a proper licensee for either extension.\\
However, the present items are not sufficient for proper constructions, which we demonstrate in the following elucidating cases:
\begin{itemize}
    \item In order to construct `vierzehn' we would merge `zehn' and `vier'. As `vier' has features left in its feature list apart from its category, a non-final merge-3 is triggered:
        \begin{flalign*}
        &\text{merge-3:}\ \dfrac{
            \langle \text{zehn} :: {=}\mathcal{C}1, +zeh, \mathcal{C}2\rangle \quad
            \langle \text{vier} :: \mathcal{C}1, -zi, -zeh, -un\rangle
          }{
           \langle \text{zehn} : +zeh, \mathcal{C}2\rangle,
            \langle \text{vier} : -zi, -zeh, -un\rangle
          }
        \end{flalign*}
        Then we would meet a dead end, where the merged items cannot perform a $zeh$-triggered move before at least the unneeded licensee $-zi$ is removed. \\
        When and how to remove $-zi$? The next case demonstrates, whether such unneeded licensees should be removed before or after the merge.
    \item For constructing `vierzig' we would merge `zig' and `vier' - again non-finally:
        \begin{flalign*}
        &\text{merge-3:}\ \dfrac{
            \langle \text{zig} :: {=}\mathcal{C}1, +zi, \mathcal{C}2\rangle \quad
            \langle \text{vier} :: \mathcal{C}1, -zi, -zeh, -un\rangle
          }{
           \langle \text{zig} : +zi, \mathcal{C}2\rangle,
            \langle \text{vier} : -zi, -zeh, -un\rangle
          }
        \end{flalign*}
        Now a move would be triggered as intended, but since `vier' has further licensees left in its feature list it would lead to a non-final move:
        \begin{flalign*}
        &\text{move-2:}\ \dfrac{
           \langle \text{zig} : +zi, \mathcal{C}2\rangle,
            \langle \text{vier} : -zi, -zeh, -un\rangle
          }{
           \langle \text{zig} : \mathcal{C}2\rangle,
            \langle \text{vier} : -zeh, -un\rangle
          }
        \end{flalign*}
        As already seen in the first example, there are licensees left to be triggered.
        Another problem is that `zig' and `vier' are still not merged finally. Since all the features - that are responsible for the correct final merge - are triggered already, a correct construction can no longer be guaranteed.\\
        Regarding the question, whether the leftover licensee should be triggered before or after the merge: If we added items that allow to remove $-zeh$ and $-un$ after the merge, we could still not guarantee that they are removed before the $zi$-triggered move. Then we still could not exclude wrong constructions. So, we rather remove unneeded features before the merge.
\end{itemize}
\textbf{Conclusion: Before a word item is merged by another, all its unneeded features have to be removed.}

In the next subsection, we explain how to remove unneeded features.
\subsection{Adapters}\label{shaping}
We call the removal of features 'adapting'. Adapting is done with adapter items. An adapter item holds a selector ${=}\mathcal{C}$ of its own category $\mathcal{C}$ and one or more licensors. This way it can merge with another item of its category $\mathcal{C}$ and trigger some of its licensees, while the category of the construction remains unchanged. Since adapter items can only adapt the items of their own category specifically, we are able to influence that items of CFG words have to be adapted before they get merged by another word's item.
The most basic way of adapting is to remove the foremost licensee $-a$ of an item of category $\mathcal{C}$ by merging with $\langle \epsilon :: {=}\mathcal{C}, +a, \mathcal{C}\rangle$ and then moving $a$-triggered. These items are called remove adapters.
\begin{instruction}
For each licensee $-a$ that appears in some item of category $\mathcal{C}$ in your lexicon, add a remove adapter $\langle\epsilon :: {=}\mathcal{C}, +a, \mathcal{C}\rangle$.
\end{instruction}

The required remove adapters for the present example are:
\begin{align*}
&\langle \epsilon1 :: {=}\mathcal{C}1, +zi, \mathcal{C}1\rangle\\
&\langle \epsilon2 :: {=}\mathcal{C}1, +zeh, \mathcal{C}1\rangle\\
&\langle \epsilon3 :: {=}\mathcal{C}1, +un, \mathcal{C}1\rangle
\end{align*}
With these remove adapters we can generate `vierundzwanzig' in the present grammar, see Figure~\ref{fig:vierundzwanzig}.
\begin{figure*}[t]
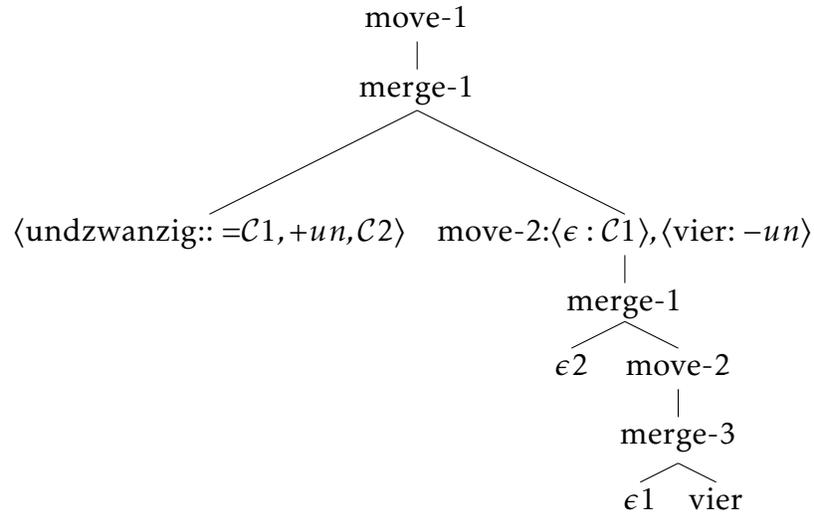

\Tree [.{move-1} [.{merge-1} [.{$\langle$undzwanzig$ :: {=}\mathcal{C}1,+un,\mathcal{C}2\rangle$} ][.{move-2:$\langle\epsilon : \mathcal{C}1\rangle,\langle$vier$ : -un\rangle$} [.{merge-1}
[.{$\epsilon2$} ][.{move-2} [.{merge-3} [.{$\epsilon1$} ][.{vier} ] ] ] ] ] ] ]
\caption{Derivation tree of `vierundzwanzig'. Since the needed licensee is the last in the feature list, `vier' can be adapted just with remove adapters. This special case causes that `vier' is merged-1 by `undzwanzig', instead of a merge-3 as in the examples in Section 3. The outcome is the same, though.}
\label{fig:vierundzwanzig}
\end{figure*}
\\
If we instead wanted to generate `vierzig' or `vierzehn', then we would need to keep a licensee - $-zi$ or $-zeh$ respectively - while removing other licensees further back - $-zeh$ and $-un$ or only $-un$ respectively. We use items called select adapters for that, which remove all licensees except for the foremost. \\
For the present example we need the following select adapters:\footnote{A select adapter $\langle\epsilon :: {=}\mathcal{C}1, +un, c1, -un\rangle$ for $-un$ is not needed, as can be seen in Figure~\ref{fig:vierundzwanzig}. However, it can be helpful to introduce one to make the constructions more uniform and thus easier adaptable for e.\,g.\ the addition of semantics.}
\begin{align*}
&\langle \epsilon :: {=}\mathcal{C}1, +zi, +zeh, +un, \mathcal{C}1, -zi\rangle\\
&\langle \epsilon :: {=}\mathcal{C}1, +zeh, +un, \mathcal{C}1, -zeh\rangle
\end{align*}
With the select adapters added we can bring `vier' into any shape in which it keeps none or one arbitrary of its licensees---except for the last, whose case is covered in Figure~\ref{fig:vierundzwanzig}---by following this procedure:
\begin{enumerate}
    \item As long as the foremost licensee of the word item is unneeded:\\
    Remove it by merging and moving with the related remove adapter.
    \item Then: If there are still unneeded licensees in the feature list: Merge with the related select adapter and perform moves until all its licensors are triggered.
\end{enumerate}
An example for a typical adapting procedure is given in Figure~\ref{fig:vier}.
\begin{figure*}[t]
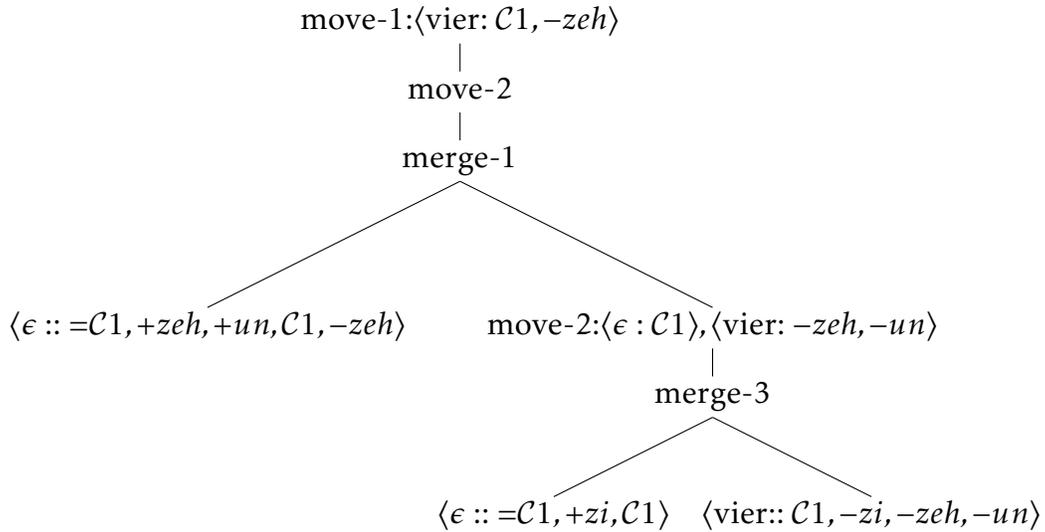

\Tree [.{move-1:$\langle $vier$ : \mathcal{C}1,-zeh\rangle$} [.{move-2} [.{merge-1} [.{$\langle\epsilon :: {=}\mathcal{C}1,+zeh,+un,\mathcal{C}1,-zeh\rangle$} ][.{move-2:$\langle\epsilon : \mathcal{C}1\rangle,\langle $vier$ : -zeh,-un\rangle$} [.{merge-3} [.{$\langle\epsilon :: {=}\mathcal{C}1,+zi,\mathcal{C}1\rangle$} ][.{$\langle $vier$ :: \mathcal{C}1,-zi,-zeh,-un\rangle$} ] ] ] ] ] ]
\caption{`vier' gets adapted in such a way that all licensees except $-zeh$ are removed, so it can construct `vierzehn'. First, $-zi$ is removed with a remove adapter. Then the select adapter for $-zeh$ removes $-un$ while keeping $-zeh$.}
\label{fig:vier}
\end{figure*}
Thus, `vier' can flexibly turn into any shape where it holds one or none of its original licensees. \\
Any adapting procedure that is finished by a final move-1 - such as in Figure~\ref{fig:vier} - gives the word item the adapted shape that is called for in the workspaces presented in Section 3. A final move occurs, when the backmost licensee of the word item is triggered. This happens in every adapting procedure, unless it is the backmost licensee that is needed to model the derivation from a restricted NT. This is showcased in Figure~\ref{fig:vierundzwanzig}. There, `vier' turns into the adapted shape $\langle\epsilon : \mathcal{C}1\rangle,\langle$vier$ : -un\rangle$. Then the first item of the chain has no more features apart from its category, so it is merged-1 finally instead of a non-final merge-3. 
After that however, constructions are the same as those showcased in Section 3.\\
\begin{instruction}\label{instr-select}
For each licensee $-x$ and each non-empty sequence $-t$ of licensees that appears behind $-x$ in the feature list of some item of some category $\mathcal{C}$, add a select adapter $\langle\epsilon::{=}\mathcal{C},+x,+t,\mathcal{C},-x\rangle$, where $+t$ denotes the sequence of the licensors corresponding to the licensees in $-t$.
\end{instruction}
Once select adapters are added, we have a grammar that allows each word item to use any single of its licensees or none of them during its derivation. 

So the desired MG would be finished.

However, there is an efficiency problem evolving from Instruction 8 which we discuss in the next subsection.
\subsection{A Sorting Problem for Efficiency}
In this subsection, we discuss a complexity issue that occurs during Instruction~\ref{instr-select}. The issue is not an essential part of this paper, but rather an outlook on how the algorithm could be optimized towards building smaller MG lexica.

Instruction~\ref{instr-select} can blow up the lexicon with select adapters for different orders of licensees%
, since the number of combinations tends to grow quickly. \\
Also, many of these select adapters would be restricted to a single use case, since a select adapter $\langle \epsilon :: {=}\mathcal{C}, +x, +y, +z, \mathcal{C}, -x\rangle$ only works for items whose feature list end in $\dots,-x, -y, -z\rangle$. Thus, we might need many select adapters just for $-x$, especially if the licensees are sorted in an unfavorable way.\\
Therefore, it is desirable to sort the licensees in such a way that the feature lists of as many word items as possible match as long as possible when read from the back.\\
For a demonstration we present the following lexicon:
\begin{align*}
&\langle \text{zehn} :: {=}\mathcal{C}1, +zeh, \mathcal{C}2\rangle \\
&\langle \text{undzwanzig} :: {=}\mathcal{C}1, +un, \mathcal{C}2\rangle\\
&\langle \text{zig} :: {=}\mathcal{C}1, +zi, \mathcal{C}2\rangle
&\vspace{0.3cm}\\
&\langle \text{eins} :: \mathcal{C}1\rangle \\
&\langle \text{zwei} :: \mathcal{C}1, -un\rangle\\
&\langle \text{drei} :: \mathcal{C}1, -zeh, -un\rangle \\
&\langle \text{vier} :: \mathcal{C}1, -zi, -zeh, -un\rangle\\
&\langle \text{fünf} :: \mathcal{C}1, -zi, -zeh, -un\rangle \\
&\langle \text{sechs} :: \mathcal{C}1, -un\rangle\\
&\langle \text{sieben} :: \mathcal{C}1, -un\rangle \\
&\langle \text{acht} :: \mathcal{C}1, -un\rangle\\
&\langle \text{neun} :: \mathcal{C}1, -zi, -zeh, -un\rangle
\end{align*}
In the given lexicon, the licensees in the feature lists of the word items (category $\mathcal{C}1$) are always sorted in such a way that the more frequent licensees are further back. This allows us to adapt all word items arbitrarily (so that exactly one arbitrary licensee or no licensee remains) with the following adapter items alone:
\begin{align*}
&\text{Remove adapters:}\\
&\langle \epsilon :: {=}\mathcal{C}1, +zi, \mathcal{C}1\rangle\\
&\langle \epsilon :: {=}\mathcal{C}1, +zeh, \mathcal{C}1\rangle\\
&\langle \epsilon :: {=}\mathcal{C}1, +un, \mathcal{C}1\rangle\\
&\text{Select adapters:}\\
&\langle \epsilon :: {=}\mathcal{C}1, +zi, +zeh, +un, \mathcal{C}1, -zi\rangle\\
&\langle \epsilon :: {=}\mathcal{C}1, +zeh, +un, \mathcal{C}1, -zeh\rangle
\end{align*}
In the present example, the sorting is simple because the item sets of the three licensees form a series of inclusions, i.\,e., $I(un)\supset I(zeh)\supset I(zi)$, given $I(a)$ denotes the set of items that hold $-a$.\\
In general, however, the item sets of the licensees can be disjoint or partially overlapping. An especially unfortunate distribution has already been found in the `Sorbian Clock' grammar in \cite{maier2022minimalist}. In this case, smart sorting becomes complicated and requires advanced mathematical methods.\label{sorting} 

\section{How Overgeneration is avoided}
In this section we give insights into why MGs constructed out of category\babelhyphen{hard}recursion free CFGs do not overgenerate.

The insights are no proof by any means, as leaving out many minor and technical aspects leads to logical flaws. The full proof can be found in \ref{proof}.

First we explain, how category-recursion can cause overgeneration.

Consider a derivation $wordX\mapsto word\alpha$ caused by the rule $X\mapsto \alpha ~\# \mathcal{C}$.
If there is a category-recursion in the grammar, $\alpha$ could (indirectly) derive another rule of category $\mathcal{C}$. In MG format this would mean that there is an MG construction on top of $\langle \underline{\alpha}::\ldots\mathcal{C},\ldots\rangle$\footnote{Here, $\underline{\alpha}$ denotes the word $\alpha$ under omission of NTs} that again has category $\mathcal{C}$. Then $\langle \underline{\alpha}::\ldots\mathcal{C}, \ldots\rangle$ is not necessarily forced to trigger all its unneeded licensees before being merged by another word's item. Instead, it could potentially trigger some of its licensee later, if the construction gets back to category $\mathcal{C}$. In the mean time, operations would not go according to plan.

An example: Consider the lexicon
\begin{align*}
    \langle apple :: \mathcal{C}NP, -t\rangle\\
    \langle \epsilon :: {=}\mathcal{C}NP, +t, \mathcal{C}NP\rangle\\
    \langle sour\_ :: {=}\mathcal{C}NP, \mathcal{C}NP\rangle
\end{align*}
Here, 'apple' is a noun phrase that for whatever purpose holds a licensee $-t$, which a remove adapter exists for, and 'sour' is an adjective that is modelled as a noun phrase requiring another noun phrase with a free NT.\\
The remove adapter can be used to remove the $-t$ from 'apple', so 'sour' can perform a merge-1 on 'apple' to form the grammatical noun phrase 'sour\_apple'.\\
However, since the category is again $\mathcal{C}NP$ after the merge, as it was before, the remove shaper can also be applied afterwards, which causes the misconstruction 'applesour\_':

\begin{figure}[H]
\centering
\Tree [.{move-1: $\langle apple\epsilon sour\_ : \mathcal{C}NP\rangle$} [.{merge-1: $\langle \epsilon sour\_ : +t, \mathcal{C}NP\rangle, \langle apple : -t\rangle$} [.{$\langle \epsilon :: {=}\mathcal{C}NP, +t, \mathcal{C}NP\rangle$} ][.{merge-3: $\langle sour\_ : \mathcal{C}NP\rangle, \langle apple : -t\rangle$} [.{$\langle sour\_ :: {=}\mathcal{C}NP, \mathcal{C}NP\rangle$}  ][.{$\langle apple :: \mathcal{C}NP, -t\rangle$} ] ] ] ]
\caption{Category recursion allows ungrammatical arrangement of exponents}
\end{figure}

'applesour\_' is ungrammatical, since it cannot be generated by the underlying CFG.\footnote{Even if probably there is a cocktail called like that}

We can avoid misconstruction, if word items trigger all licensees of the mergees directly after the merge.

The present discussion contains the most important argument for the following auxiliary result:
\begin{theorem*}\label{main-aux-th}
Let $G$ be a MG that is constructed out of a category-recursion free CFG following the instructions in this paper. Let $l$ be a legal expression of $G$. Let $f,x$ be lexical or derived item of $G$, so that a merge of $f$ with $x$ is part of the construction of $l$.
Then $f$ and $x$ have the shapes
\begin{align*}
&x=\langle \ldots\cdot \mathcal{C},-t_1,\ldots,t_m\rangle\\
&f=\langle \ldots\cdot {=}\mathcal{C},+t_1,\ldots,+t_m,a,\ldots\rangle,
\end{align*}
where $a$
\begin{itemize}
\item is either a category, selector or licensor, but it
\item is not a licensor, if $\mathcal{C}(f)\neq\mathcal{C}(x)$.
\end{itemize}
\end{theorem*}
The theorem mirrors the results of Theorem~\ref{merge-criteria} and it is proven in \ref{proof}.
It requires many auxiliary results beforehand to prove that each licensee has to be triggered after a word item gets merged. The graph in Figure \ref{fig:aux-th-graph} shows headlines of all used auxiliary results, arranged in a graph of dependencies:
\begin{figure}
    \centering
    \includegraphics[scale = 0.52]{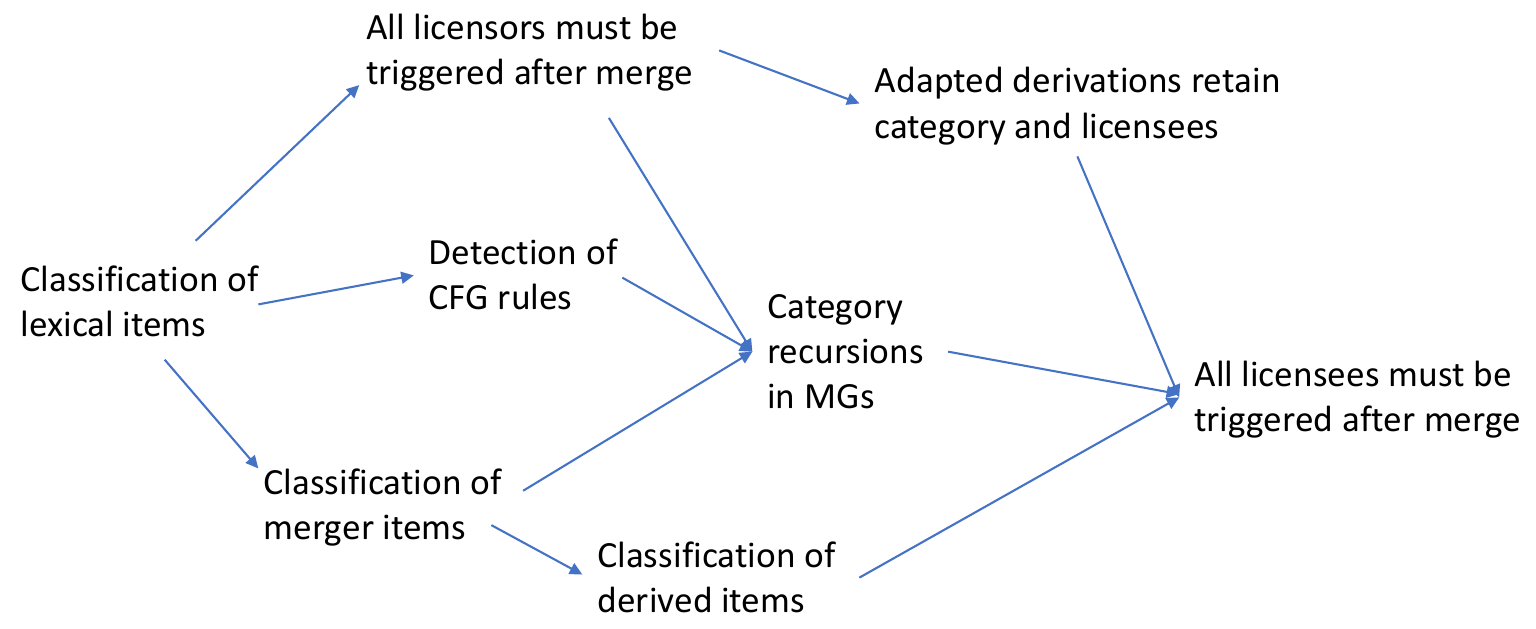}
    \caption{Graph of auxiliary results (represented by headlines) used to prove Theorem~\ref{merge-criteria}.3. The arrows indicate that each result used all parent results in order to be proven.}
    \label{fig:aux-th-graph}
\end{figure}

Next, we describe, how it follows from the auxiliary result that a MG constructed out of a category-recursion free CFG does not overgenerate.

We call the feature sequence ${=}\mathcal{C},+t_1,\ldots,+t_m$ the foremost selector\babelhyphen{hard}licensor block of $f$,
The theorem can be interpreted so that
\begin{itemize}
\item a merger $f$'s foremost selector-licensor block needs to correspond to the start of its mergee $x$'s feature list and
\item when a mergee $x$ gets merged by a merger with $\mathcal{C}(f)\neq\mathcal{C}(x)$, $f$'s foremost selector-licensor block has to match $x$'s entire feature list exactly.
\end{itemize}
Note that $\mathcal{C}(f)\neq\mathcal{C}(x)$ whenever $f$ represents a CFG word, otherwise there would be a category-recursion in the CFG.
Considering the equivalence
\begin{align*}
f\text{ is not an adapter}\Leftrightarrow f\text{ is a word item }\Leftrightarrow \mathcal{C}(f)\neq\mathcal{C}(x),
\end{align*}
we may rephrase the meaning of Theorem~\ref{main-aux-th} to:
An item $x$ can be merged by
\begin{itemize}
\item a shaper item, if the shapers foremost selector-licensor block corresponds to $x$'s foremost features
\item a word item, if the word item's foremost selector-licensor block corresponds to $x$'s entire feature list.
\end{itemize}

Based on these interpretations, we can fully describe, by what kind of merger items a mergee item of $x$'s shape can be merged until it is merged by a word item (see Figure\ref{fig:possibility}).

\begin{figure}[ht]
\centering
\includegraphics[scale = 0.6]{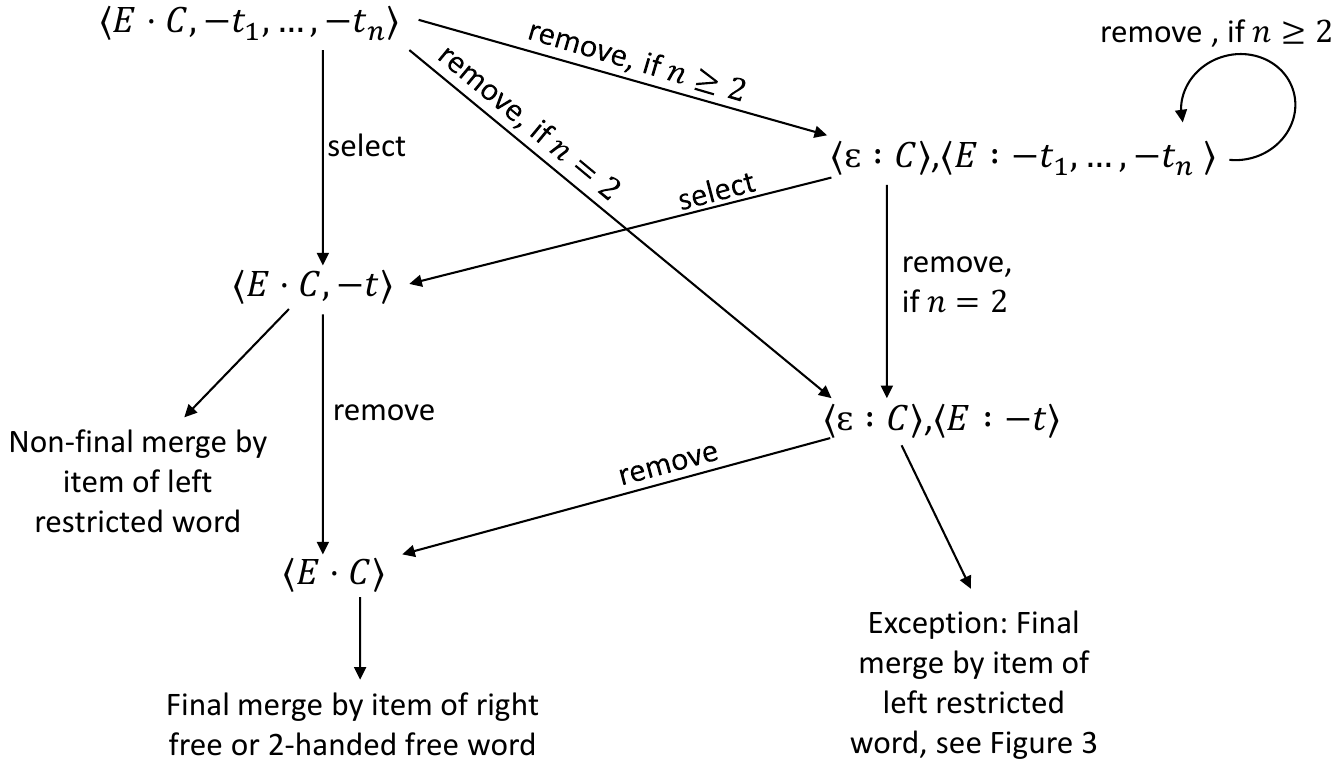}
\caption{All possible legal paths of derivation from an item of shape $\langle E\cdot \mathcal{C}, -t_1,\dots,-t_n\rangle$ or $\langle \epsilon :\mathcal{C}\rangle, \langle E:-t_1,\dots,-t_n\rangle$. Each arrow represents a merge- and potential consecutive move-operation(s) performed on the item that the arrow goes out of. The markers remove/select indicate that the arrow's merge- and move-operation(s) are triggered by a remove/select adapter.}\label{fig:possibility}
\end{figure}
It can be seen that merges by word's items only happen in such ways as they are described in Subsection~\ref{basic-wordshapes}. This way, we can show that misconstructions such as 'applesour\_' do not appear.

\section{Summary and Results}\label{summary}
We have presented the following algorithm to turn a CFG with categories into a MG:
\begin{enumerate}
    \item Add rule $S_0\mapsto S ~\# \mathcal{C}Fin$ to the CFG and reset the start symbol to $S_0$
    \item Normalize each NT's target category:
    \begin{itemize}
        \item[] For each NT $A$:
        \begin{itemize}
            \item[] If $A$ has rules of more than one category:
            \begin{itemize}
                \item[] Choose main category $\mathcal{C}_m$ so that neither of the other categories have a word containing a NT that has a rule of $\mathcal{C}_m$.
                \item[] For all other categories $\mathcal{C}_i$ of $A$:
                \begin{itemize}
                    \item[] introduce a new auxiliary NT $A_i$ and add a rule $A\mapsto A_i ~\# \mathcal{C}_m$
                    \item[] replace all rules $A\mapsto \ldots ~\# \mathcal{C}_i$ by $A_i\mapsto \ldots: \mathcal{C}_i$
                \end{itemize}
            \end{itemize}
        \end{itemize}
    \end{itemize}
    \item Classify all NTs as free or restricted with respect to Definition~\ref{def-free}
    \item Decompose each rule
    \begin{align*}
        A\mapsto s_0X_1s_1X_2\ldots X_ns_n ~\# \mathcal{C}
    \end{align*}
    into several rules
    \begin{align*}
    &A&\mapsto s_0&A': &&&&&&&&\mathcal{C}\\
    &&&            A'&\mapsto X_1s_1&A'': &&&&&&\mathcal{C}'\\
    &&&&&                            \vdots\\
    &&&&&                            A^{(i)}&\mapsto X_is_i&A^{(i+1)}: &&&&\mathcal{C}^{(i)}\\
    &&&&&&&                                                 \vdots\\
    &&&&&&&                                                 A^{(n)}&&\mapsto X_ns_n: &&\mathcal{C}^{(n)},
    \end{align*}
    under usage of the following exceptions:
    \begin{itemize}
        \item If $s_0$ is $\epsilon$, summarize the first two rules to $A\mapsto X_1s_1A''$.
        \item If $s_n$ is $\epsilon$ and $X_n$ is free, summarize the last two rules  to $A^{(n-1)}\mapsto X_{n-1}s_{n-1}X_n: \mathcal{C}^{(n-1)}$.
        \item For $i=1,\ldots,n-1$, if $X_i$ is restricted, replace the rule $A^{(i)}\mapsto X_is_iA^{(i+1)}: \mathcal{C}^{(i+1)}$ by two rules $A^{(i)}\mapsto X'_is_iA^{(i+1)}: \mathcal{C}^{(i+1)}$ and $X'_i\mapsto X_i: \mathcal{C}'_i$.
        \item If $X_n$ is free and $s_n$ is not $\epsilon$, then replace $A^{(n)}\mapsto X_ns_n$ by two rules $A^{(n)}\mapsto X_ns_nO: \mathcal{C}^{(n)}$ and $O\mapsto \epsilon: \mathcal{C}nix$.
    \end{itemize}
    \item For each right free word $wordX$ of category $\mathcal{C}$ create a MG item
    \begin{align*}
        \langle word :: {=}\mathcal{C}(X), \mathcal{C}\rangle.
    \end{align*}
    For each 2-handed free word $YwordX$ of category $\mathcal{C}$ create a MG item
    \begin{align*}
        \langle word :: {=}\mathcal{C}(X), {=}\mathcal{C}(Y), \mathcal{C}\rangle.
    \end{align*}
    For each left restricted word $Yword$ of category $\mathcal{C}$ create a MG item
    \begin{align*}
        \langle word :: {=}\mathcal{C}(Y), +a_Y, \mathcal{C}\rangle.
    \end{align*}
    For each terminal word $word$ of category $\mathcal{C}$ create a MG item
    \begin{align*}
        \langle word :: \mathcal{C}\rangle.
    \end{align*}
    \item Give to each MG item $\langle word :: \ldots \mathcal{C}\rangle$ - which represents the CFG word $wordX$, $YwordX$, $Yword$ or $word$ in category $\mathcal{C}$ - one licensee $-a_R$ for each restricted NT $R$ that generates its CFG word.
    \item For each licensee $-a$ that appears in some item of category $\mathcal{C}$ in your lexicon, add a remove adapter $\langle\epsilon :: {=}\mathcal{C}, +a, \mathcal{C}\rangle$
    \item For each licensee $-x$ and each non-empty sequence $-t$ of licensees that appears behind $-x$ in the feature list of some item of some category $C$, add a select adapter $\langle\epsilon::{=}\mathcal{C},+x,+t,\mathcal{C},-x\rangle$, where $+t$ denotes the sequence of the licensors corresponding to the licensees in $-t$.
\end{enumerate}

\vspace{0.3cm}
The implementation can be found in \cite{implementation}.\\
In \cite{maier2022minimalist} we uploaded MG lexica that are designed based on the ideas of the presented algorithm. They generate numerals words in Deutsch, English, French, Mandarin and Upper Sorbian, as well as date and time of day expressions in Upper Sorbian. For the Upper Sorbian language sets (word classes), we also created FST lexica and discussed the pros and cons of FSTs and MGs in Sections 2.1. and 2.2. of \cite{maier-EtAl:2022:EURALI}. We concluded that MGs can work with significantly smaller lexica than FSTs do, especially if the grammar becomes more complicated. In particular, MGs can deal more efficiently with iteration and with dependency between non-adjacent morphemes. Table~\ref{tab:lexsizes} 
shows sizes of MG lexica and some FST and CFG lexica and underlines our conclusions on a small sample size.

\begin{table}[ht]
\centering
\begin{tabular}{ l | r | r | r}
Language & MG & FST & CFG \\
\hline
English Numerals $<10^6$ & 40 & - & 39 \\
Deutsch Numerals $<10^6$ & 56 & - & 75 \\
French Numerals $<10^6$ & 66 & - & 81 \\
Mandarin Numerals $<10^8$ & 63 & - & 53 \\
Upper Sorbian Numerals $<10^{15}$ & 131 & 356 & 188 \\
Upper Sorbian Times of Day & 133  & 1479 & 75 \\
Upper Sorbian Dates & 133 & 118 & 210 \\
YYYY-MM-DD & 56 & - & 82 \\
\end{tabular}
\caption{Lexicon sizes for different languages in MG, FST and CFG format}
\label{tab:lexsizes}
\end{table}
For the MG lexica the sorting problem of licensees has been dealt with manually.
As can be seen, MGs can often operate on similar or smaller sized lexicon than CFGs, with 2 exceptions:
\begin{itemize}
    \item For the Mandarin numerals many words with inner inputs were used, e.g. for $E,\textit{qiān},O,P,D\mapsto^* \textit{èr qiān líng yī shí yī}$ (by word: two-thousand-zero-one-eleven). For MG format the rule would have to be decomposed in many rules in Instruction 4, which causes the MG lexicon to be bigger than the CFG lexicon. However, many CFG applications - e.g. the Cocke-Younger-Kasami algorithm - ask for a CFG normal form anyway that would not allow inner NTs either.
    \item The Sorbian Clock is an unfortunate case where the licensees have an especially high degree of mutual overlap. This is because the licensees control the agreement between number words for time and daytime words, e.g. 'at nine in the morning' is desired, while 'at nine in the afternoon' is not. This agreement naturally has much overlap, since the boundaries between e.g. morning and noon are blurred
\end{itemize}

\section{Outlook}
We presented an algorithm on how to transform a CFG into a MG that generates a desired language without overgeneration.
In~\cite{RLMNG} semantic terms were introduced into MGs. So, their inclusion into the algorithm is naturally desired.

Apart from Stabler's original MG style, there are also further implementations of the Minimalist Program which the presented algorithm could be adapted to, see~\cite{graf2012movement} and~\cite{stabler2011computational}. We assume that this can work again by determining the word shapes that are easiest to MG model and then use a decomposition into those shapes, similar as in Instruction 4.

Since MGs have been compared with multiple context-free grammars in the past, see \cite{meryMCFG2006}, a similar MCFG-to-MG algorithm may be concievable.

In Subsection~\ref{sorting}, we presented a basic idea how to ideally sort the licensees in the feature lists, such that less select adapters are required in the lexicon. 
More detailed ideas have yet been applied in the construction of the `Sorbian Clock' grammar in~\cite{maier2022minimalist}. The ideas can be extended to a general sorting algorithm. The goal of the sorting algorithm would be to sort the licensees in such a way that as many given feature lists as possible match as long as possible when read from the back.

The usual questions about runtime and effect of such an algorithm can then be included in the evaluation of the utility: How much more useful is a lexicon with a few powerful (select adapter) items compared to a lexicon with a mass of single-use (select adapter) items? How do the resulting MGs compare to other grammar formats like, e.\,g., Finite State Transducers?

As the proposed instructions are ideal for grammars without recursions, a good field of application is numeral grammars, as they have a clear hierarchy.
In \cite{MaierEUNICE} we have proposed an algorithm to decompose numerals. Based on the determined compositions, numeral systems of over 250 natural languages can be presented as a CFG and then transformed into a MG. The producable grammars can also be automatically equipped with number semantics. This way, valuable resources can be produced for speech technology of peculiar languages.

The algorithm facilitates the production of a bigger database of minimalist grammars, which may even evade overgeneration. MGs constructed with the present algorithm offer an explanation for each feature. Since MGs naturally also have a strong distinction between code and lexicon, they are especially well-suited for explainable AI methods, such as language learning algorithms, potentially even for spontaneous speech.

\bibliographystyle{elsarticle-harv}
\bibliography{CFG2MG}

\begin{thebibliography}{35}
\expandafter\ifx\csname natexlab\endcsname\relax\def\natexlab#1{#1}\fi
\providecommand{\url}[1]{\texttt{#1}}
\providecommand{\href}[2]{#2}
\providecommand{\path}[1]{#1}
\providecommand{\DOIprefix}{doi:}
\providecommand{\ArXivprefix}{arXiv:}
\providecommand{\URLprefix}{URL: }
\providecommand{\Pubmedprefix}{pmid:}
\providecommand{\doi}[1]{\href{http://dx.doi.org/#1}{\path{#1}}}
\providecommand{\Pubmed}[1]{\href{pmid:#1}{\path{#1}}}
\providecommand{\bibinfo}[2]{#2}
\ifx\xfnm\relax \def\xfnm[#1]{\unskip,\space#1}\fi
\bibitem[{Amblard(2011)}]{amblard2011minimalist}
\bibinfo{author}{Amblard, M.}, \bibinfo{year}{2011}.
\newblock \bibinfo{title}{Minimalist grammars and minimalist categorial
  grammars, definitions toward inclusion of generated languages}.
\newblock \href{http://arxiv.org/abs/1108.5096}{{\tt arXiv:1108.5096}}.
\bibitem[{Chomsky(1995)}]{Chomsky1995}
\bibinfo{author}{Chomsky, N.}, \bibinfo{year}{1995}.
\newblock \bibinfo{title}{The Minimalist Program}.
\newblock \bibinfo{publisher}{MIT Press}.
\newblock \DOIprefix\doi{10.7551/mitpress/9780262527347.001.0001}.
\bibitem[{Copestake(2002)}]{copestake2002implementing}
\bibinfo{author}{Copestake, A.}, \bibinfo{year}{2002}.
\newblock \bibinfo{title}{Implementing typed feature structure grammars}.
  volume \bibinfo{volume}{110}.
\newblock \bibinfo{publisher}{CSLI publications Stanford}.
\bibitem[{Englisch(2019)}]{englischimplementing}
\bibinfo{author}{Englisch, J.}, \bibinfo{year}{2019}.
\newblock \bibinfo{title}{Implementing minimalist syntax and remove}.
\newblock \bibinfo{journal}{Structure Removal} .
\bibitem[{Ermolaeva(2020)}]{ermolaeva2020induction}
\bibinfo{author}{Ermolaeva, M.}, \bibinfo{year}{2020}.
\newblock \bibinfo{title}{Induction of minimalist grammars over morphemes}.
\newblock \bibinfo{journal}{Proceedings of the Society for Computation in
  Linguistics} \bibinfo{volume}{3}, \bibinfo{pages}{484--487}.
\newblock \DOIprefix\doi{10.7275/jed1-5j57}.
\bibitem[{Ermolaeva(2021)}]{ermolaeva-2021-deconstructing}
\bibinfo{author}{Ermolaeva, M.}, \bibinfo{year}{2021}.
\newblock \bibinfo{title}{Deconstructing syntactic generalizations with
  minimalist grammars}, in: \bibinfo{booktitle}{Proceedings of the 25th
  Conference on Computational Natural Language Learning},
  \bibinfo{publisher}{Association for Computational Linguistics},
  \bibinfo{address}{Online}. pp. \bibinfo{pages}{435--444}.
\newblock \URLprefix \url{https://aclanthology.org/2021.conll-1.34},
  \DOIprefix\doi{10.18653/v1/2021.conll-1.34}.
\bibitem[{Fowlie and Koller(2017)}]{Fowlie2017}
\bibinfo{author}{Fowlie, M.}, \bibinfo{author}{Koller, A.},
  \bibinfo{year}{2017}.
\newblock \bibinfo{title}{Parsing minimalist languages with interpreted regular
  tree grammars}, in: \bibinfo{booktitle}{Proceedings of the 13th International
  Workshop on Tree Adjoining Grammars and Related Formalisms}, pp.
  \bibinfo{pages}{11--20}.
\bibitem[{Graben et~al.(2019)Graben, Römer, Meyer, Huber and Wolff}]{RLMNG}
\bibinfo{author}{Graben, P.b.}, \bibinfo{author}{Römer, R.},
  \bibinfo{author}{Meyer, W.}, \bibinfo{author}{Huber, M.},
  \bibinfo{author}{Wolff, M.}, \bibinfo{year}{2019}.
\newblock \bibinfo{title}{Reinforcement learning of minimalist numeral
  grammars}, in: \bibinfo{booktitle}{2019 10th IEEE International Conference on
  Cognitive Infocommunications (CogInfoCom)}, pp. \bibinfo{pages}{67--72}.
\newblock \DOIprefix\doi{10.1109/CogInfoCom47531.2019.9089924}.
\bibitem[{Graf(2012)}]{graf2012movement}
\bibinfo{author}{Graf, T.}, \bibinfo{year}{2012}.
\newblock \bibinfo{title}{Movement-generalized minimalist grammars}, in:
  \bibinfo{booktitle}{International Conference on Logical Aspects of
  Computational Linguistics}, \bibinfo{organization}{Springer}. pp.
  \bibinfo{pages}{58--73}.
\newblock \DOIprefix\doi{10.1007/978-3-642-31262-5_4}.
\bibitem[{Graf(2017)}]{graf2017}
\bibinfo{author}{Graf, T.}, \bibinfo{year}{2017}.
\newblock \bibinfo{title}{A computational guide to the dichotomy of features
  and constraints}.
\newblock \bibinfo{journal}{Glossa: a journal of general linguistics}
  \bibinfo{volume}{2}, \bibinfo{pages}{18}.
\newblock \DOIprefix\doi{10.5334/gjgl.212}.
\bibitem[{Harkema(2001)}]{harkema}
\bibinfo{author}{Harkema, H.}, \bibinfo{year}{2001}.
\newblock \bibinfo{title}{A characterization of minimalist languages}, in:
  \bibinfo{editor}{de~Groote, P.}, \bibinfo{editor}{Morrill, G.},
  \bibinfo{editor}{Retor{\'e}, C.} (Eds.), \bibinfo{booktitle}{Logical Aspects
  of Computational Linguistics}, \bibinfo{publisher}{Springer},
  \bibinfo{address}{Berlin, Heidelberg}. pp. \bibinfo{pages}{193--211}.
\newblock \DOIprefix\doi{10.1007/3-540-48199-0_12}.
\bibitem[{Herring(2016)}]{herring2016grammar}
\bibinfo{author}{Herring, J.}, \bibinfo{year}{2016}.
\newblock \bibinfo{title}{Grammar construction in the Minimalist Program}.
\newblock Ph.D. thesis. Indiana University.
\bibitem[{Hopcroft and Ullman(1979)}]{HopcroftUllman1979IATLC}
\bibinfo{author}{Hopcroft, J.E.}, \bibinfo{author}{Ullman, J.D.},
  \bibinfo{year}{1979}.
\newblock \bibinfo{title}{Introduction to Automata Theory, Languages and
  Computation}.
\newblock \bibinfo{publisher}{Addison-Wesley}.
\bibitem[{Indurkhya(2019)}]{indurkhya2019}
\bibinfo{author}{Indurkhya, S.}, \bibinfo{year}{2019}.
\newblock \bibinfo{title}{Automatic inference of minimalist grammars using an
  smt-solver}.
\newblock \bibinfo{journal}{CoRR} \bibinfo{volume}{abs/1905.02869}.
\newblock \URLprefix \url{http://arxiv.org/abs/1905.02869},
  \DOIprefix\doi{10.48550}, \href{http://arxiv.org/abs/1905.02869}{{\tt
  arXiv:1905.02869}}.
\bibitem[{Jäger and Rogers(2012)}]{JaegerRogers2012F}
\bibinfo{author}{Jäger, G.}, \bibinfo{author}{Rogers, J.},
  \bibinfo{year}{2012}.
\newblock \bibinfo{title}{Formal language theory: refining the chomsky
  hierarchy}.
\newblock \bibinfo{journal}{Philosophical Transactions of the Royal Society B:
  Biological Sciences} \bibinfo{volume}{367}, \bibinfo{pages}{1956–1970}.
\newblock \URLprefix \url{http://dx.doi.org/10.1098/rstb.2012.0077},
  \DOIprefix\doi{10.1098/rstb.2012.0077}.
\bibitem[{Kobele(2010)}]{kobele2010}
\bibinfo{author}{Kobele, G.M.}, \bibinfo{year}{2010}.
\newblock \bibinfo{title}{Without remnant movement, mgs are context-free}, in:
  \bibinfo{editor}{Ebert, C.}, \bibinfo{editor}{J{\"a}ger, G.},
  \bibinfo{editor}{Michaelis, J.} (Eds.), \bibinfo{booktitle}{The Mathematics
  of Language}, \bibinfo{publisher}{Springer Berlin Heidelberg},
  \bibinfo{address}{Berlin, Heidelberg}. pp. \bibinfo{pages}{160--173}.
\bibitem[{Kobele(2011)}]{kobele2011minimalist}
\bibinfo{author}{Kobele, G.M.}, \bibinfo{year}{2011}.
\newblock \bibinfo{title}{Minimalist tree languages are closed under
  intersection with recognizable tree languages}, in:
  \bibinfo{booktitle}{International Conference on Logical Aspects of
  Computational Linguistics}, \bibinfo{organization}{Springer}. pp.
  \bibinfo{pages}{129--144}.
\newblock \DOIprefix\doi{10.1007/978-3-642-22221-4_9}.
\bibitem[{Kobele(2014)}]{kobele2014meeting}
\bibinfo{author}{Kobele, G.M.}, \bibinfo{year}{2014}.
\newblock \bibinfo{title}{Meeting the boojum}.
\newblock \bibinfo{journal}{Theoretical Linguistics} \bibinfo{volume}{40},
  \bibinfo{pages}{165--173}.
\newblock \DOIprefix\doi{10.1515/tl-2014-0007}.
\bibitem[{Kobele(2021)}]{kobele2021}
\bibinfo{author}{Kobele, G.M.}, \bibinfo{year}{2021}.
\newblock \bibinfo{title}{Minimalist grammars and decomposition}.
\newblock \bibinfo{note}{Submitted}.
\bibitem[{Maier(2023)}]{implementation}
\bibinfo{author}{Maier, I.}, \bibinfo{year}{2023}.
\newblock \bibinfo{title}{{CFG-to-MG-algorithm}}.
\newblock \URLprefix
  \url{https://github.com/ikmMaierBTUCS/CFG-to-MG-algorithm}.
\bibitem[{Maier and Kuhn(2022)}]{maier2022minimalist}
\bibinfo{author}{Maier, I.}, \bibinfo{author}{Kuhn, J.}, \bibinfo{year}{2022}.
\newblock \bibinfo{title}{{Minimalist Grammars - Numerals and expressions for
  date and time of day}}.
\newblock \URLprefix
  \url{https://github.com/ikmMaierBTUCS/Minimalist-Grammars/}.
\bibitem[{Maier et~al.(2022)Maier, Kuhn, Duckhorn, Kraljevski, Sobe, Wolff and
  Tschöpe}]{maier-EtAl:2022:EURALI}
\bibinfo{author}{Maier, I.}, \bibinfo{author}{Kuhn, J.},
  \bibinfo{author}{Duckhorn, F.}, \bibinfo{author}{Kraljevski, I.},
  \bibinfo{author}{Sobe, D.}, \bibinfo{author}{Wolff, M.},
  \bibinfo{author}{Tschöpe, C.}, \bibinfo{year}{2022}.
\newblock \bibinfo{title}{Word class based language modeling: A case of upper
  sorbian}, in: \bibinfo{booktitle}{Proceedings of The Workshop on Resources
  and Technologies for Indigenous, Endangered and Lesser-resourced Languages in
  Eurasia within the 13th Language Resources and Evaluation Conference},
  \bibinfo{publisher}{European Language Resources Association},
  \bibinfo{address}{Marseille, France}. pp. \bibinfo{pages}{28--35}.
\newblock \URLprefix \url{https://aclanthology.org/2022.eurali-1.5}.
\bibitem[{Maier and Wolff(2022)}]{MaierEUNICE}
\bibinfo{author}{Maier, I.}, \bibinfo{author}{Wolff, M.}, \bibinfo{year}{2022}.
\newblock \bibinfo{title}{Poster: Decomposing numerals}, in:
  \bibinfo{booktitle}{EUNICE Science Dissemination: Poster Competition}.
\newblock \URLprefix
  \url{https://www.b-tu.de/en/communications-engineering/publications/papers\#c314785},
  \DOIprefix\doi{10.5281/zenodo.7501698}.
\bibitem[{Mery et~al.(2006)Mery, Amblard, Durand and Retoré}]{meryMCFG2006}
\bibinfo{author}{Mery, B.}, \bibinfo{author}{Amblard, M.},
  \bibinfo{author}{Durand, I.}, \bibinfo{author}{Retoré, C.},
  \bibinfo{year}{2006}.
\newblock \bibinfo{title}{A Case Study of the Convergence of Mildly
  Context-Sensitive Formalisms for Natural Language Syntax: from Minimalist
  Grammars to Multiple Context-Free Grammars}.
\newblock \bibinfo{type}{Technical Report}. INRIA.
\bibitem[{Michaelis(2001)}]{Michaelis2001OnFP}
\bibinfo{author}{Michaelis, J.}, \bibinfo{year}{2001}.
\newblock \bibinfo{title}{On Formal Properties of Minimalist Grammars}.
  volume~\bibinfo{volume}{13} of \textit{\bibinfo{series}{Linguistics in
  Potsdam}}.
\newblock \bibinfo{publisher}{Universitätsbibliothek, Publikationsstelle},
  \bibinfo{address}{Potsdam}.
\bibitem[{Stabler(1997)}]{stabler97}
\bibinfo{author}{Stabler, E.}, \bibinfo{year}{1997}.
\newblock \bibinfo{title}{Derivational minimalism}, in:
  \bibinfo{editor}{Retor{\'e}, C.} (Ed.), \bibinfo{booktitle}{Logical Aspects
  of Computational Linguistics}, \bibinfo{publisher}{Springer Berlin
  Heidelberg}, \bibinfo{address}{Berlin, Heidelberg}. pp.
  \bibinfo{pages}{68--95}.
\newblock \DOIprefix\doi{10.1007/BFb0052152}.
\bibitem[{Stabler(2011a)}]{stabler-2011-top}
\bibinfo{author}{Stabler, E.}, \bibinfo{year}{2011}a.
\newblock \bibinfo{title}{Top-down recognizers for {MCFG}s and {MG}s}, in:
  \bibinfo{booktitle}{Proceedings of the 2nd Workshop on Cognitive Modeling and
  Computational Linguistics}, \bibinfo{publisher}{Association for Computational
  Linguistics}, \bibinfo{address}{Portland, Oregon, USA}. pp.
  \bibinfo{pages}{39--48}.
\newblock \URLprefix \url{https://aclanthology.org/W11-0605}.
\bibitem[{Stabler(2013)}]{Stabler2013}
\bibinfo{author}{Stabler, E.}, \bibinfo{year}{2013}.
\newblock \bibinfo{title}{Two models of minimalist, incremental syntactic
  analysis}.
\newblock \bibinfo{journal}{Topics in cognitive science} \bibinfo{volume}{5}.
\newblock \DOIprefix\doi{10.1111/tops.12031}.
\bibitem[{Stabler(2011b)}]{Stabler2011AGBT}
\bibinfo{author}{Stabler, E.P.}, \bibinfo{year}{2011}b.
\newblock \bibinfo{title}{After government and binding theory}, in:
  \bibinfo{editor}{van Benthem, J.F.A.K.}, \bibinfo{editor}{ter Meulen, A.}
  (Eds.), \bibinfo{booktitle}{Handbook of Logic and Language}.
  \bibinfo{edition}{2nd} ed.. \bibinfo{publisher}{Elsevier}, p.
  \bibinfo{pages}{395–414}.
\newblock \URLprefix
  \url{http://dx.doi.org/10.1016/B978-0-444-53726-3.00007-4},
  \DOIprefix\doi{10.1016/b978-0-444-53726-3.00007-4}.
\bibitem[{Stabler(2011c)}]{stabler2011computational}
\bibinfo{author}{Stabler, E.P.}, \bibinfo{year}{2011}c.
\newblock \bibinfo{title}{Computational perspectives on minimalism}.
\newblock \bibinfo{journal}{Oxford handbook of linguistic minimalism} ,
  \bibinfo{pages}{617--643}\DOIprefix\doi{10.1093/oxfordhb/9780199549368.013.0027}.
\bibitem[{Stabler and Keenan(2003)}]{STABLER2003345}
\bibinfo{author}{Stabler, E.P.}, \bibinfo{author}{Keenan, E.L.},
  \bibinfo{year}{2003}.
\newblock \bibinfo{title}{Structural similarity within and among languages}.
\newblock \bibinfo{journal}{Theoretical Computer Science}
  \bibinfo{volume}{293}, \bibinfo{pages}{345--363}.
\newblock \URLprefix
  \url{https://www.sciencedirect.com/science/article/pii/S0304397501003516},
  \DOIprefix\doi{https://doi.org/10.1016/S0304-3975(01)00351-6}.
  \bibinfo{note}{algebraic Methods in Language Processing}.
\bibitem[{Stanojevi{\'c} and Stabler(2018)}]{Stanojevic2018}
\bibinfo{author}{Stanojevi{\'c}, M.}, \bibinfo{author}{Stabler, E.},
  \bibinfo{year}{2018}.
\newblock \bibinfo{title}{A sound and complete left-corner parsing for
  minimalist grammars}, in: \bibinfo{booktitle}{Proceedings of the Eight
  Workshop on Cognitive Aspects of Computational Language Learning and
  Processing}, pp. \bibinfo{pages}{65--74}.
\newblock \DOIprefix\doi{10.18653/v1/W18-2809}.
\bibitem[{Torr(2019)}]{Torr2019a}
\bibinfo{author}{Torr, P.}, \bibinfo{year}{2019}.
\newblock \bibinfo{title}{Wide-coverage statistical parsing with minimalist
  grammars}.
\newblock Ph.D. thesis. University of Edinburgh.
\bibitem[{Versley(2016)}]{Versley2016}
\bibinfo{author}{Versley, Y.}, \bibinfo{year}{2016}.
\newblock \bibinfo{title}{Discontinuity (re) $^2$-visited: A minimalist
  approach to pseudoprojective constituent parsing}, in:
  \bibinfo{booktitle}{Proceedings of the Workshop on Discontinuous Structures
  in Natural Language Processing}, pp. \bibinfo{pages}{58--69}.
\newblock \DOIprefix\doi{10.18653/v1/W16-0907}.
\bibitem[{Vijay-Shanker and Weir(1994)}]{Vijay-ShankerWeir1994}
\bibinfo{author}{Vijay-Shanker, K.}, \bibinfo{author}{Weir, D.J.},
  \bibinfo{year}{1994}.
\newblock \bibinfo{title}{The equivalence of four extensions of context-free
  grammars}.
\newblock \bibinfo{journal}{Mathematical Systems Theory} \bibinfo{volume}{27},
  \bibinfo{pages}{511–546}.
\newblock \URLprefix \url{http://dx.doi.org/10.1007/BF01191624},
  \DOIprefix\doi{10.1007/bf01191624}.

\end{thebibliography}

\clearpage
\appendix
\section{Proof of Functionality and Exactness}\label{proof}
In this section,
\begin{itemize}
    \item $J$ is an arbitrary CFG with categories,
    \item $I$ is the CFG with categories constructed out of $J$ by Instructions 1-4 in Section~\ref{summary} and
    \item $G$ is an MG constructed out of $I$ by the remaining instructions.
\end{itemize}
Remember that the terms 'word' and 'rule' always refer to CFGs and 'item' and 'operation' always refer to MGs.

We show that every expression generatable by $J$ is generatable by $G$ and that - if $J$ is category-recursion free - the reverse is also true.

First, we show that $I$ and $J$ have the same language.

\begin{theorem}
    $I$ and $J$ have the same language.
\end{theorem}
\begin{proof}
    By Instruction 4, $I$ is made out of $J$ by seperating rules of the shape $A\mapsto s_0X_1s_1\ldots X_ns_n$ into many rules. (Also, some rules are united to one in the instruction, but this only affects rules that have been united in $J$ but intermediately separated in the process)

    It is obvious that the separated rules in $I$ combined can generate anything that the united rule in $J$ can. So, $L(J)\subseteq L(I)$.
    We need to show the reverse.

    Consider a legal derivation in $I$ that uses a rule that does not exist in $J$. Then this rule has been created in Instruction 4 in order to represent the $i$th terminal string of a word $\alpha$ from a rule $A\mapsto\alpha$ that exists in $J$.\footnote{It could also be $S_0\mapsto S$, but this rule has no impact on the language.} 
    We call the rule $A_i\mapsto\alpha_i$. We show that it is inevitable then that all separated parts of $A\mapsto \alpha$ are used, so $A\mapsto \alpha$ gets reconstructed in any case.

    Note that $A_i\mapsto \alpha_i$ contains at least one auxiliary NT, i.e. either $\alpha_i$ contains one or $A$ is one (or both). Auxiliary NTs are unique, i.e. they have only one occurrence in words and only one rule. So,
    \begin{itemize}
        \item if $A_i$ is auxiliary then there is one rule $A_{i-1}\mapsto \alpha_{i-1}$ where $\alpha_{i-1}$ contains $A_i$ and
        \item if $\alpha_i$ contains the auxiliary NT $A_{i+1}$ there is one rule $A_{i+1}\mapsto \alpha_{i+1}$.
    \end{itemize}
    By studying Instruction 4, one can notice that the rule(s) $A_{i-1}\mapsto \alpha_{i-1}$ and/or $A_{i+1}\mapsto \alpha_{i+1}$ are also separated parts of $\overline{A}\mapsto \overline{\alpha}$.
    So, by induction the abovementioned properties of $A_i\mapsto\alpha_i$ can also be applied on $A_j\mapsto\alpha_j$ for any $j$th component rule of $A\mapsto\alpha$.
    Thus, the derivation can only be continued with a unique combinable rule that also evolved from the separation of $A\mapsto \alpha$. The derivation path of the generation and derivation of auxiliary NTs around $A_i\mapsto \alpha_i$ cannot use other rules than those constructed in Instruction 4 until $A^0=A$ is not auxiliary at the upper end and until $\alpha_n=\alpha$ contains no auxiliary NTs at the lower end. Then the derivation around $A_i\mapsto \alpha_i$ comprises all rules of the separation of $A\mapsto \alpha$ and combined they reproduce $A\mapsto \alpha$.
    Thus, any rule $A_i\mapsto \alpha_i$ in $I$, separated from a united rule $A\mapsto \alpha$ in $J$ cannot generate anything beyond what $A\mapsto \alpha$ can. Hence, $L(I)\subseteq L(J)$ and $L(J)=L(I)$.

\end{proof}
From now it suffices to show the equivalence of $I$ and $G$.

\begin{theorem}
Each legal expression of $I$ is a legal expression of $G$.
\end{theorem}
\begin{proof}
A construction in $I$ can be presented as a tree (see Figure~\ref{fig:tree}).

\begin{figure}[ht]
\centering
\includegraphics[scale = 0.52]{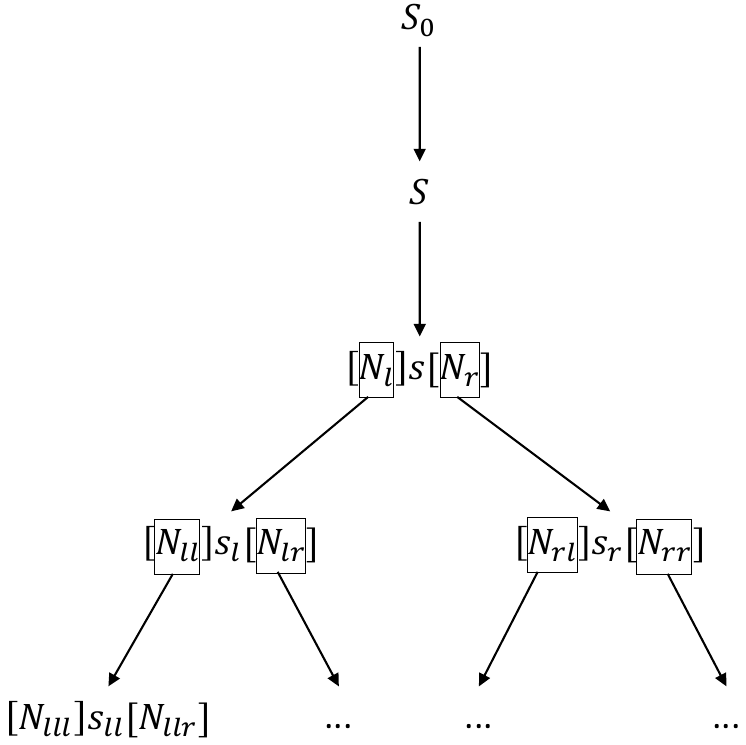}\caption{Here, each edge represents a derivation rule $N_w\mapsto [N_{wl}]s_w[N_{wr}]$, with a series $w\in\{r,l\}^*$ of indexes, $s_w$ a string of terminals and $N_{wl}$ and $N_{wr}$ the potential left and right NTs. Consequently, each edge represents a CFG word and has one outgoing edge for each of the word's NTs. This way the graph reflects the hierarchy of the resulting MG items}\label{fig:tree}
\end{figure}

The leaves of the tree have no NTs. So, following Instruction 5, they are initially created as simple MG items $\langle X::\mathcal{C}_X\rangle$ with their word/exponent $X$ and category $\mathcal{C}_X$. By Instruction 6 they may gain a list of licensees, so they finally have the shape $\langle X::\mathcal{C}_X, -a_1,\ldots,-a_n\rangle$ in the MG lexicon. All words in the tree's lowest tier are of this shape. In the second lowest tier, there is at least one word $[N_{wl}]s_w[N_{wr}]$ containing at least one NT, otherwise there would not be a tier below.

Let $F$ be such a word of 2nd lowest tier and let $X$ be the leaf word that is derived from $F$'s NT. In case $F$ is a two-handed free word, let $X$ represent the two leaf words $(X_1,X_2)$ that are derived from $F$'s NTs. The shape of $F$ is $\langle F :: {=}\mathcal{C}_X, \mathcal{C}_F\rangle$ or $\langle F :: {=}\mathcal{C}_{X1}, {=}\mathcal{C}_{X2}, \mathcal{C}_F\rangle$ or $\langle F :: {=}\mathcal{C}_X, +a_X, \mathcal{C}_F\rangle$, if $F$ is right-free or 2-handed-free or left-restricted respectively.

The MG items of $F$ and $X$ have the same shape as the respective items in Section~\ref{basic-wordshapes} - $F$ being the merger and $X$ the mergee(s) respectively - except that each of them may hold more licensees in their feature list than required.

If $F$ has such a list of licensees, it would not impact the construction. In the structure-building rules shown in Subsection 2.1 one can see that the merger or mover may have a chain ($t$ or $t_1$) that would persist any merge- or move-operation.

If (one of the components of) $X$ with the shape $\langle X:: \mathcal{C}_X,-a_1,\ldots,-a_n\rangle$ has an overshoot of licensees, it can be adapted as described in Section 4. This way it can get an adapted shape $\langle X:: \mathcal{C}_X\rangle$ or $\langle X:: \mathcal{C}_X,-a_i\rangle$ for $i\lneq n$, so the construction would work as intended.

In order to use the licensee $-a_n$ one can use $n-1$ remove adapters to produce the shape $\langle\epsilon: \mathcal{C}_X\rangle, \langle X: -a_n\rangle$. Then the construction can be continued as in Figure~\ref{fig:vierundzwanzig} and again produce the intended output. In its way, it deviates slightly from the cases in Section~\ref{basic-wordshapes}, as a final merge1 on $\langle\epsilon: \mathcal{C}_X\rangle, \langle X: -a_n\rangle$ replaces a non-final merge3 on $\langle X: \mathcal{C}_X, -a_n\rangle$. However, both operations lead to the same outcome.

In order to use another licensee $-a_i$, one can use $i-1$ remove shapers and one select shaper to produce the shape $\langle\epsilon: \mathcal{C}_X\rangle, \langle X: -a_i\rangle$.

In order to use no licensee, one can use $n$ remove shapers to produce the shape $\langle\epsilon: \mathcal{C}_X\rangle, \langle X: -a_i\rangle$.

In either case, a shape could be produced that can undergo the same operations as in Section~\ref{basic-wordshapes}.
Thus, the derivation of leaf words from the NT(s) of an item of 2nd lowest tier can be performed as a construction in MG format.

Once the derivation is performed a derived MG item $\langle F(X):\mathcal{C}_F, a_1,\ldots,a_n\rangle$ is obtained with the exponent $F(X)\in\{FX,XF,X_2FX_1\}$.
The derived item $F(X)$ has the same shape as a leaf item, except that it is derived. This difference cannot matter anymore, as it would only split the difference between triggering a merge1 or a merge2. But $F(X)$ does not have any selectors to trigger a merge, so it behaves just like a leaf item in advance. By recursion, we can conclude that derivations of 3rd lowest tier words' NTs to 2nd lowest tier words - and generally $n$th tier words' NTs to $(n+1)$th tier words - can also be performed as intended.

Hence each arbitrary construction in $I$ - as presented in the tree - is also constructable in $G$.
\end{proof}

Before we can state and prove the reverse of this theorem, we need many auxiliary results.

\subsection{Auxiliary Results}
\subsubsection{Classification of Lexical Items in \(G\) and Merge Criteria}
In order to delimit what $G$ could potentially generate, we aim to classify what merges can happen in the construction of a legal expression in $G$. Merges are the operations that make the difference, while Moves are just automatic consequences of Merges.

First, we profile the lexical items of $G$ based on the instructions:

\begin{observation}\label{itemshape}
Each lexical item in $G$ has the shape profile:
\begin{align*}
    \langle E :: A_1,\ldots,A_n, \mathcal{C}, -t_1,\ldots,-t_m\rangle,
\end{align*}
where
\begin{itemize}
    \item $m,n\in\mathbb{N}_0$, $n\leq 2$
    \item $E$ is the exponent
    \item for $i=1,\dotsc,n$: $A_i$ is a selector-licensor-block ${=}\mathcal{C}_i, +a_{i,1},\dotsc,+a_{i,l(i)}$, where
    \begin{itemize}
        \item $l(i)\in\mathbb{N}_0$
        \item ${=}\mathcal{C}_i$ is the selector of some category $\mathcal{C}_i$
        \item $+a_{i,1},\dotsc,+a_{i,l(i)}$ are licensors
    \end{itemize}
    \item $\mathcal{C}$ is the category
    \item $-t_1,\dotsc,-t_m$ are licensees
\end{itemize}
\end{observation}
More specifically, we can differentiate:
\begin{observation}[Classification of lexical items]\label{profiling}
Each lexical item in $G$ belongs to one of the types listed in the following table: 
\begin{table}[H]
\begin{longtable}{ | m{2cm}| l | m{5.3cm}| }
\hline
Type & Shape & Notable properties\\
\hline
Word items &  $\langle E:: [A_1],[A_2],\mathcal{C},-t_1,\dotsc,-t_m\rangle$ & \makecell[l]{$n\leq 2$\\ $n=1\Rightarrow l(1)\leq 1$\\ $n=2\Rightarrow l(1)=l(2)=0$\\ Represents a word in $I$}\\
\hline
Remove adapters & $\langle \epsilon :: {=}\mathcal{C}, +a, \mathcal{C}\rangle$ & $E=\epsilon$, $n=1$, $\mathcal{C}_1=\mathcal{C}$, $l(1)=1$ and $m=0$\\
\hline
Select adapters & $\langle \epsilon :: {=}\mathcal{C}, +a_1,\dotsc,+a_l, \mathcal{C}, -a_1\rangle$ & $E=\epsilon$, $n=1$, $\mathcal{C}_1=\mathcal{C}$, $l(1)\geq 2$, $m=1$ and $-t_1$ corresponds to $+a_{1,1}$.\\
\hline
\end{longtable}\caption{Classification of lexical items}\label{table:class}
\end{table}
\label{tab:MGitemtypes}
The variables $m,n\ldots$ in Table~\ref{table:class} refer to the notation in Observation~\ref{itemshape}.
\end{observation}
\begin{note}
    The item $\langle \epsilon :: \mathcal{C}nix\rangle$ is also profiled as a word item in Observation~\ref{profiling}, as it represents the word $\epsilon$ of the rule $O\mapsto \epsilon ~\#\mathcal{C}nix$.
\end{note}
Next, we analyse under which conditions two items can merge.
\begin{theorem}\label{merge-criteria}
Let $f,x$ be lexical or derived\footnote{We used notation $\cdot$ for either '$::$' (lexical) or '$:$' (derived), see Sect. \ref{stablermgs}} items of $G$ with the shape
\begin{align*}
    f=\langle E \cdot A_1,\ldots,A_n ,\mathcal{C}_f ,-t_1,\ldots,-t_m\rangle\\
    x=\langle E' \cdot A'_1,\ldots,A'_{n'} ,\mathcal{C}_x ,-t'_1,\ldots,-t'_{m'}\rangle
\end{align*}
with
\begin{align*}
    A_j={=}\mathcal{C}_j,+a_{j,1},\ldots,+a_{j,l(j)}\text{ for }j=1,\ldots,n\\
    A'_j={=}\mathcal{C}'_j,+a'_{j,1},\ldots,+a'_{j,l'(j)}\text{ for }j=1,\ldots,n'.
\end{align*}
If $f$ and $x$ merge - as merger and mergee respectively - at any point of the construction of a legal derivation tree in $I$, then
\begin{enumerate}
    \item $n'=0$ and $\mathcal{C}_1=\mathcal{C}_x$.
    \item $m'\geq l(1)$ and $a_{1,i}=t'_i\ \forall i\leq l(1)$. 
    \item if $I$ is category-recursion free and $\mathcal{C}_f\neq \mathcal{C}_x$, then $m'=l(1)$ 
\end{enumerate}
\end{theorem}
\begin{proof}
\begin{enumerate}
    \item If $n'\neq 0$, then there is a selector in front of the category, which would have to be triggered before the merge. If $\mathcal{C}_1\neq \mathcal{C}_x$, then the merger's selector and the mergee's category do not correspond, so a merge would be illegal.
    \item A merge of $f$ with $x$ produces
    \begin{align*}
    \langle E: +a_{1,1},\ldots,+a_{1,l(1)}, A_2,\ldots,A_n,\mathcal{C}_f,-t_1,\ldots,-t_m\rangle,\langle E':-t'_1,\ldots,-t'_{m'}\rangle.
    \end{align*}
    Then $+a_{1,1}$ would trigger a move with a corresponding $-a_{1,1}$. If $-t'_1\neq -a_{1,1}$, the move could not be performed, so the construction would meet a dead end and no legal derivation tree could be formed. Thus $-t'_1=-a_{1,1}$ and a move would be performed. After that, a corresponding licensee for $+a_{1,2}$ would be needed, which must be $-t'_2$. By induction, $-t'_i$ has to correspond to $a_{1,i}$ for $i=1,\ldots,l(1)$, so $l(1)\leq m'$.
    \item {\bf Proven later} at the end of \ref{CLFinG}.
\end{enumerate}
\end{proof}

We will prove Statement 3 of Theorem~\ref{merge-criteria} at the end of \ref{CLFinG}. First, we need more definitions and auxiliary results. All auxiliary results used to prove it are also arranged in the dependency graph in Figure~\ref{fig:aux-th-graph}.  


\subsubsection{Classification of Derived Items in \(G\)}
Recently, we have classified lexical items in $G$. Next, we go ahead and classify derived items. We only consider derived items upon which a legal expression is derivable, while we might not always mention this restriction.

\begin{definition}[Self-Triggered and (Purely) Externally Triggered Derivation]
Let $x$ be a (lexical or derived) item in $G$ and let $y$ be an item derived from $x$, including the possibility $y=x$.
\begin{itemize}
    \item $y$ is a \emph{self-triggered} derivation of $x$, if all merges involved in the derivation are triggered by $x$'s selectors. If $p_1,\ldots,p_n$ with $n\leq 2$ are the mergees of the self-triggered derivation chronologically, we write
    \begin{align*}
        y=x(p_1,\ldots,p_n).
    \end{align*}
    \item $y$ is a \emph{purely externally triggered} derivation of $x$, if no merges involved in the derivation are triggered by $x$'s selectors. Let $m_1,\ldots,m_n$ be the - lexical or derived - mergers of the purely externally triggered derivation chronologically. If all $m_i$ have only one selector, then we can write
    \begin{align*}
        y= m_n\bigr(\ldots m_2\bigr(m_1(x)\bigl)\bigl).
    \end{align*}
    However, if some $m_i$ has 2 selectors, then $m_{i+1}$ cannot merge it, until $m_i$ has triggered its 2nd selector to merge with some co-mergee $q_i$. So, generally we need to write
    \begin{align}\label{pureextern}
        y= m_n\bigl(\ldots \bigl(m_2\bigl(m_1(x,[q_1]),[q_2]\bigr),[q_3]\bigr)\ldots\bigr).
    \end{align}
    By defining
    \begin{align*}
        m^q(x):=m(x,q),
    \end{align*}
    we can shorten Equation~\ref{pureextern} to
    \begin{align*}
        y = m_n^{[q_n]}\bigl(...m_1^{[q_1]}(x)\bigr)\text{ or}\\
        y = m_n^{[q_n]}\circ...\circ m_1^{[q_1]}(x),
    \end{align*}
    where $m^{[q]}(x)$ implies that $q$ may not necessarily exist.
    \item $y$ is an \emph{externally triggered} derivation of $x$, if not all merges involved in the derivation are triggered by $x$'s selectors. If chronologically $m_1,\ldots,m_n$ are the mergers, $q_i$ is the potential co-mergee of $m_i$ and $p_1,\ldots,p_k$ are the mergees merged by $x$
    , we write
    \begin{equation}\label{genDeriva}
    \begin{aligned}
        y = m_n^{[q_n]}\circ...\circ m_1^{[q_1]}\bigl(x(p_1,...,p_k)\bigr)
    \end{aligned}
    \end{equation}
\end{itemize}
Generally, from now we will often often write $\overline{y}$ to name an item that $y$ is derived from and $\hat{x}$ for an item derived from $x$.
\end{definition}


\begin{note}
    For a - lexical or derived - item $x$ in $G$
    \begin{itemize}
        \item any derivation $y$ can be presented as in Equation~\ref{genDeriva} with $n,k\in\mathbb{N}_0$. It follows that any derived item $y$ is a self-triggered derivation of some item $\overline{y}$. In the equation, $\overline{y}$ is $m_n$, if $n>0$, and $\overline{y}=x$ otherwise.
        \item a self-triggered derivation of $x$ has the same category and licensees as $x$.
    \end{itemize}
\end{note}
\begin{definition}[Adapted derivation]
    An adapted derivation is an externally triggered derivation, where all mergers are adapter items.
\end{definition}
\begin{lemma}[Adapted derivations retain category and licensees]\label{shapedderiva}
    Let $x$ be an item and $y$ an adapted derivation of $x$ that is part of a legal derivation tree. Then, item $x$ has the same category as $y$ and $y$'s set of licensees is a subset of the licensees of $x$.
\end{lemma}
\begin{proof}
    An adapted derivation can involve self-triggered operations and operations triggered by adapters. Self-triggered operations have no impact on category or licensees.

    Now it suffices to prove the theorem for $y$ being a purely externally triggered adapted derivation of $x$. Since all mergers are adapters, none of them has 2 selectors, so none of them merges a co-mergee. Thus $y=m_n\circ\ldots\circ m_1(x)$ with adapter items $m_1,...,m_n$.

    It suffices to prove the statement for just $n=1$ adapter item, so $y=m_1(x)$, then it follows by induction that any number $n\in\mathbb{N}$ of adapter items will never change the category or add new licensees.

    $m_1(x)$ has the same category as $x$, as adapters can only merge items of the category they hold, so they cannot change the category.

    Can $m_1$ bring additional licensees to $x$?

    Remove adapters have no licensees so they cannot enrich $x$'s list of licensees.

    Select adapters have one licensee $-a$, but they also hold the corresponding $+a$ as their foremost licensor in their only selector-licensor block. By Statement 2 of Theorem~\ref{merge-criteria} the select adapter can only merge $x$, if $x$ also has $-a$ as its foremost licensee. Then $x$'s $-a$ would be triggered by the adapters $+a$ in the consecutive move. So, select adapters do not enrich the list of licensees either.
\end{proof}
\begin{lemma}[Classification of merger items in $G$]\label{classmerg}
    When an item $x$ in $G$ triggers a merge during the construction of a legal derivation tree, then $x$ is either an adapter item or a word item or a self-triggered derivation of a word item.
\end{lemma}
\begin{proof}
    Let $x$ be an item in $G$. $x$ needs a selector to trigger its merge. If $x$ is lexical, then by Observation~\ref{profiling} it can only be an adapter item or a word item.

    If $x$ is derived, then it is a self-derivation of some lexical $\overline{x}$. The lexical item $\overline{x}$ has at least two selectors, one for the merge to become $x$ and one for the merge that $x$ triggers. So, since only word items can have more than one selector, $x$ can only be a self-derivation of a word item $\overline{x}$, if it is derived.
\end{proof}
\begin{lemma}[Classification of derived items in $G$]\label{classderiva}
    Each derived item $x$ in $G$ - that is not lexical - is an adapted derivation of a lexical word item.
\end{lemma}
\begin{proof}
    Let $x$ be a non-lexical item. First, note that $x$ is derived from some word item. As $x$ is lexical, its construction involves an initial merge with a lexical mergee, which by Statement 1 of Theorem~\ref{merge-criteria} cannot have any selector. By Observation~\ref{profiling}, the initial mergee can only be a word, which $x$ is eventually derived from.

    Let $w$ be a lexical word item, which $x$ is derived from. Then $w$ has a self-triggered derivation $\hat{w}=w(\ldots)$ so that
    \begin{align*}
        x=m_n^{[q_n]}\circ\ldots\circ m_1^{[q_1]}(\hat{w})
    \end{align*}
    with mergers $m_1,\ldots,m_n$ and co-mergees $q_i$ for all $m_i$ with 2 selectors. 


    Each $m_i$ (first) merges $m_{i-1}^{[q_{i-1}]}\circ\ldots\circ m_1^{[q_1]}(x)$, regardless whether or not it has a co-mergee $q_i$. So by Lemma~\ref{classmerg}, each $m_i$ is either an adapter or (a self-triggered derivation of) a word $\overline{m_i}$. If all $m_i$ are adapters, then $x=m_n\circ\ldots\circ m_1(\hat{w})$ is an adapted derivation of the lexical word $w$. Otherwise, let $z$ be the largest index such that $\overline{m_{z}}$ is a word item. Then $m_{z+1},...,m_{n}$ are adapter items and $m_{z}$ is a self-triggered derivation of $\overline{m_{z}}$, so
    \begin{align*}
        x=m_n\circ\ldots\circ m_{z+1}\bigl(\underbrace{m_{z}\bigl(m_{z-1}(\ldots),[q_z]\bigr)}_{=\overline{m_z}(\ldots)}\bigr)
    \end{align*}
    is an adapted derivation of the lexical word $\overline{m_{z}}$.
\end{proof}
\begin{definition}[Head Word]
    The head word of a derived item in $G$ is the lexical word item of which it is an adapted derivation. The definition may also be used for a lexical word item, which is the head word of itself.
\end{definition}


\subsubsection{Impact of Category-Recursion in \(I\) on \(G\)}\label{CLFinG}
In this section we show that, if $I$ is category-recursion free, $G$ has a similar property.
\begin{lemma}[Detection of CFG rules]\label{input-criterion}
    Let $\phi$ and $\chi$ be words (right-hand sides) of CFG rules in $I$ and let $\mathcal{C}_{\phi}$ and $\mathcal{C}_{\chi}$ be their categories. Let
    \begin{align*}
        F=\langle f::A_{1},...,A_{n},\mathcal{C}_{\phi},\ldots \rangle\\
        \text{and }X=\langle x::\ldots \mathcal{C}_\chi,-t_1,\ldots,-t_m\rangle
    \end{align*}
    with $n\leq 2$
    be the MG items created with Instruction 5 representing the CFG words $\phi$ and $\chi$. Their exponents $f$ and $x$ equal $\phi$ and $\chi$ under omission of NTs respectively. 

    If one of the $A_{i}$ equals either ${=}\mathcal{C}_{\chi}$ or ${=}\mathcal{C}_{\chi},+t_j$ for some $j\leq m$, then there is a non-terminal $N$ that generates $\chi$ and is contained in the word $\phi$.
\end{lemma}
\begin{proof}
    Assume that some $A_i$ in $F$ is ${=}\mathcal{C}_{\chi}$ without a licensor. Then reviewing Instruction 5, $F$ has gotten ${=}\mathcal{C}_{\chi}$ due to $\phi$ being right free or 2-handed free, so $\phi$ holds a free NT $N$ with target category $\mathcal{C}_{\chi}$. As $N$ is free, it generates any CFG word of category $\mathcal{C}_{\chi}$, so there is a rule $N\mapsto \chi~\#\mathcal{C}_{\chi}$.

    Assume that some $A_i$ in $F$ is ${=}\mathcal{C}_{\chi},+t_j$. Then reviewing Instruction 5, $F$ has gotten ${=}\mathcal{C}_{\chi}$ due to $\phi$ being left restricted, so $\phi$ holds a NT $N$ with target category $\mathcal{C}_{\chi}$ that is restricted with the licensor $+t_j$. Since $N$ is restricted it can only generate some of the CFG words of category $\mathcal{C}_{\chi}$. In the MG the word items of those CFG words are given a licensee $-t_j$ in Instruction 6. $\chi$'s MG item $X$ has category $\mathcal{C}_{\chi}$ and a licensor $+t_j$, so it is one of those that can be generated by $N$, so there is a rule $N\mapsto \chi ~\#\mathcal{C}_{\chi}$.
\end{proof}
Now we can interpret the impact of category-recursions in $I$ on $G$.
\begin{lemma}[Category-Recursions in MGs]\label{MG-cat-recursion}
Let $I$ be category-recursion free and let $a$ be an item in $G$ (derived or lexical) and $b$ an item derived from $a$, such that $\mathcal{C}(a)=\mathcal{C}(b)$ and $b$ is part of a legal derivation tree. Then $b$ is an adapted derivation of $a$.
\end{lemma}
\begin{note}
    The consequence in Lemma \ref{MG-cat-recursion} can be interpretted as:
    \begin{itemize}
        \item Item $a$ has a self-triggered derivation $\hat{a}$, so that $\hat{a}$ and $b$ have the same exponent. Or
        \item During the derivation from $a$ to $b$, the category has remained the same all the time.
    \end{itemize}
\end{note}
\begin{proof}
Each self-triggered derivation of $a$ is an adapted derivation of $a$.
Therefore let $b$ be an externally triggered derivation of $a$.

Let $\hat{a}$ be the self-triggered derivation of $a$ without selectors or licensors, so that $b$ is a derivation of $\hat{a}$.



Let $\overline{a}$ be the head word of $\hat{a}$. Then $\overline{a}$ is also the head word of $a$. Note that $\hat{a}$, $a$ and $\overline{a}$ have the same category.

We have
\begin{align*}
    b=m_n^{[q_n]}\circ\ldots\circ m_1^{[q_1]}(\hat{a})
\end{align*}
with merger items $m_1,\ldots,m_n$ and co-mergees $q_i$.
We need to show that all $m_i$ are adapter items. 

By Lemma~\ref{classmerg}, each $m_i$ is either an adapter or (a self-triggered derivation of) a lexical word item.

Rename the $m_1,\ldots,m_n$ and the $q_i$ so that
\begin{align*}
b=d_{p,k(p)}\circ\ldots\circ d_{p,1}\circ &w_p^{[q_p]}\circ d_{p-1,k(p-1)}\circ\ldots\\&\ldots\\\ldots\circ d_{2,1}\circ &w_2^{[q_2]}\circ d_{1,k(1)}\circ\ldots\\ \ldots\circ d_{1,1}\circ&w_1^{[q_1]}\circ d_{0,k(0)}\circ\ldots\circ d_{0,1}(\hat{a})
\end{align*}
where for each $i$ the $w_i$ is a self-triggered derivations of a word item $\overline{w_i}$, $q_i$ is $w_i$'s potential co-mergee and the $d_{i,j}$ are the $k(i)$ adapter items used between $w_i^{[q_i]}$ and $w_{i+1}^{[q_{i+1}]}$.

Set $w_0=\hat{a}$ and $\overline{w_0}=\overline{a}$ and
\begin{align*}
    r_i = d_{i,k(i)}\circ\ldots\circ w_i^{[r_i]}\circ\ldots\circ d_{0,1}(\hat{a})
\end{align*}
for $i=0,\ldots,p$. Then $b=r_p$ and for $i=0,...,p-1$: $r_i$ is a mergee that the merger $w_{i+1}$ merges with in the construction of $b$.

Note that for $i=1,\ldots,p$,
\begin{itemize}
    \item $w_i$ is a self-triggered derivation of $\overline{w_i}$, so $\overline{w_i}$ has all selector-licensor-blocks that $w_i$ has,
    \item since $r_{i-1}$ is merged by $w_i$, by Statement 2 of Theorem~\ref{merge-criteria} the foremost selector-licensor block of $w_i$ corresponds to the category (and foremost licensee) of $r_{i-1}$ and
    \item $r_{i-1}$ 
    is an adapted derivation of $\overline{w_{i-1}}$, so by Lemma~\ref{shapedderiva} $\overline{w_{i-1}}$ has the category and all the licensees that $r_{i-1}$ has.
\end{itemize}
Hence, one selector-licensor-block of $\overline{w_i}$ corresponds to the category (and a licensee) of $\overline{w_{i-1}}$, so by Lemma~\ref{input-criterion}, the CFG word of $\overline{w_i}$ contains a NT that generates the CFG word of $\overline{w_{i-1}}$ for $i=1,\ldots,p$.

Now remember that
\begin{itemize}
    \item $b$ is $r_p$, which is an adapted derivation of $\overline{w_p}$, so $\mathcal{C}(b)=\mathcal{C}(\overline{w_p})$,
    \item $a$ has the same category as $\overline{a}$, and since $\overline{a}=\overline{w_0}$, we have $\mathcal{C}(a)=\mathcal{C}(\overline{w_0})$ and
    \item $\mathcal{C}(a)=\mathcal{C}(b)$, so we see that $\mathcal{C}(\overline{w_0})=\mathcal{C}(a)=\mathcal{C}(b)=\mathcal{C}(\overline{w_p})$, and the CFG items of $\overline{w_0},\ldots,\overline{w_p}$ would form a category-recursion in $I$, if $p>0$.
\end{itemize}
Since by assumption $I$ has no category-recursions, we know that $p=0$, so all mergers $m_1,\ldots,m_n$ are adapter items. 
\end{proof}

Now we have all tools to prove Statement 3 of Theorem~\ref{merge-criteria}.
\begin{proof}
We showed that Statement 2 is true because every licensor needs a corresponding licensee to merge with. Now we show that - if $I$ is category-recursion free and $\mathcal{C}\neq \mathcal{C}'$ - there cannot be more licensees than licensors.

If $m'>l(1)$, then after the merge and the moves triggered by the $+a_{1,1},\ldots,+a_{1,l(1)}$, we would have an item $g=f(x)$ of category $\mathcal{C}_f$ with at least one licensee $-t'_{m'}$, which from now we call $-t$. In order for this item to become a legal expression, $-t$ eventually needs to be triggered by some $+t$.

We show that this is no longer possible.

First we ask, in which case $x$ holds $-t$. If $x$ is lexical, it has to be a word item. Otherwise, as an adapter $x$ would hold a selector, so by Statement 1 of Theorem~\ref{merge-criteria} item $x$ could not be merged by $f$. If $x$ is derived, then by Lemma~\ref{classderiva} it is an adapted derivation of a word item $\overline{x}$. By Lemma~\ref{shapedderiva}, $\overline{x}$ then also holds $-t$. Summarized, $x$ is an adapted derivation of some word $\overline{x}$ of category $\mathcal{C}_x$ that holds $-t$. The items $x$ and $\overline{x}$ may be equal. 

Following Instruction 6, the purpose of giving $\overline{x}$ a licensee $-t$ is to allow $\overline{x}$ to be merged by some word item, that represents a CFG word containing a restricted NT for the (main) target category $\mathcal{C}_x$. Only items of the main target category $\mathcal{C}_x$ can hold $-t$.

Following Instructions 5-6, a corresponding $+t$ for $-t$ can then only be found behind a selector ${=}\mathcal{C}_x$, let it be in a word or an adapter.

Recall that $g=f(x)$ and $\mathcal{C}(g)=\mathcal{C}_f\neq \mathcal{C}_x$ and $f$ holds $-t$. 
In order for $g$ to get its $-t$ triggered, it needs a derivation $h$ of category $\mathcal{C}_x$ that could then be merged by a word or adapter with a selector-licensor block ${=}\mathcal{C}_x, +t$. Then $h$ is a derivation of $\overline{x}$ with $\mathcal{C}(h)=\mathcal{C}(\overline{x})$, so by Lemma~\ref{MG-cat-recursion}, it is an adapted derivation. 
This means that the category cannot have deviated from $\mathcal{C}(\overline{x})$ at any point in the derivation between $\overline{x}$ and $h$.
However, $f$ is an intermediate derivation of category $\mathcal{C}_f\neq \mathcal{C}_x=\mathcal{C}(\overline{x})$.

Hence, such a derivation $h$ cannot exist and $-t'_{m'}$ can no longer be removed after $x$ is merged by $f$, so a legal expression can no longer be formed.
\end{proof}
\subsection{Exactness}

We have shown that $G$ generates everything that $I$ does.

Now we use the auxiliary results - mainly Theorem~\ref{merge-criteria} - to show that $G$ generates nothing beyond what $I$ does, if $I$ has no category-recursions.

\begin{theorem}
If $I$ is category-hierarchical, then each legal expression of $G$ is a legal expression of $I$.
\end{theorem}
\begin{proof}
Consider a legal derivation tree in $G$. We show that in any path from a leaf to the root, any operation that affects the exponent is analogous to one the operations described in Subsection~\ref{basic-wordshapes} or Figure~\ref{fig:vierundzwanzig}.

Any path of legal derivation tree in $G$ includes at the bottom an initial merge of a lexical mergee item, i.e. a lexical item without selectors and licensors. Let $x$ be such an item. Then each (legal) derivation of $x$ can be written
\begin{align}\label{derivoflexmergee}
    m_s^{[q_s]}\circ\ldots\circ m_1^{[q_1]}(x)
\end{align}
where the $m_i$ are the - lexical or derived - mergers in chronological order and the $q_i$ are the co-mergees for those $m_i$ that have 2 selectors.

We show that whenever an intermediate operation from \\$m_i^{[q_i]}\circ\ldots\circ m_1^{[q_1]}(x)$ to $m_{i+1}^{[q_{i+1}]}\circ\ldots\circ m_1^{[q_1]}(x)$ affects the exponent, it behaves like in Section~\ref{basic-wordshapes} or Figure~\ref{fig:vierundzwanzig}.
Since the mergee $x$ has no selectors, it has the shape $\langle E:: \mathcal{C},-t_1,\ldots,-t_l\rangle$, so by Observation~\ref{tab:MGitemtypes}, it has to be a word.
\begin{itemize}
    \item[] For the sake of generalization, we already note that it would not make a difference if it had either of the shapes
    \begin{equation}
    \begin{aligned}
        \langle E:: \mathcal{C},-t_1,\ldots,-t_l\rangle\\
        \langle E: \mathcal{C},-t_1,\ldots,-t_l\rangle\\
        \langle \epsilon: \mathcal{C}\rangle ,\langle E:-t_1,\ldots,-t_l\rangle
    \end{aligned}
    \label{mergee-shapes}
    \end{equation}
    They produce the exact same derivations once they get merged by some $\langle E'\cdot {=}\mathcal{C},\ldots\rangle$. In the latter case, the merge would be final merge-1 instead of a non-final merge-3, but in either case the merge would produce $\langle E': \ldots\rangle ,\langle E:-t_1,\ldots,-t_l\rangle$.
\end{itemize}
The merger, by Lemma~\ref{classderiva}, can be an adapter or a (self-triggered derivation of a) word item.\\
If the merger is an adapter:
\begin{itemize}
    \item[] By Theorem~\ref{merge-criteria}.2 it is the remove adapter or the select adapter of the first licensee $-t_1$.
    If it is
    \begin{itemize}
        \item[] the\footnote{There is a highly unlikely possibility that a select adapter with $n'<n$ fitting licensors exists and could be applied. This situation would imply a very bad sorting of licensees, albeit possibly unavoidable. In this case, a merge-3 and $m$ move-2s would produce $\langle\epsilon:\mathcal{C},-t_1\rangle,\langle E:-t_{n'+1},\ldots,-t_n\rangle$. This item would also behave like the items in Equation~\ref{mergee-shapes}} select adapter, one merge-3, $n-1$ times move-2 and a move-1 produce the shape $\langle E: \mathcal{C},-t_1\rangle$. As this is a special case of one of the mergee shapes in Equation~\ref{mergee-shapes}, the consecutive derivations can be described the same again.
        \item[] the remove adapter and
        \begin{itemize}
            \item[] $l=1$, then a merge-3 and a move-1 produce $\langle E: \mathcal{C}\rangle$
            \item[] $l>1$, then a merge-3 and a move-2 produce $\langle \epsilon :\mathcal{C}\rangle ,\langle E: -t_2,\ldots,-t_l\rangle$.
        \end{itemize}
    \end{itemize}
    In either case the product would be a special case of one of the mergee shapes in Equation~\ref{mergee-shapes}, so the consecutive derivations can be described the same again.
\end{itemize}
If the merger is (a self-derivation of) a word item, then the merger has a different category as the mergee, otherwise there would be a category recursion in the CFG. So, by Theorem~\ref{merge-criteria}.3, we have:

{\bf The merger can only be (a self-derivation of) a word if the first selector-licensor block in its feature list exactly matches the mergee's feature list.} \\
More precisely, then the number $l$ of the merger's licensors is always $0$ or $1$, depending on whether the word's NT is free or restricted. So, the mergee can only be merged by a (a self-derivation of) a word if it is $\langle E\cdot \mathcal{C}\rangle$ or $\langle E\cdot \mathcal{C},-t\rangle$ or $\langle\epsilon :\mathcal{C}\rangle ,\langle E: -t\rangle$.\\
With the present results we can completely describe\footnote{Application of a select shaper with less than $n$ licensors not included in the graph. Then $\langle\epsilon:\mathcal{C},-t_1\rangle,\langle E:-t_{n'+1},\ldots,-t_n\rangle$ would be produced, a construction that would re-enter the graph latest when it has only one licensor left.} in Figure~\ref{pathsofadaptedderivation} what a lexical item - or more generally: an item without selectors and licensors - could derive until it gets merged by another (self-derivation of a) word item.
\begin{figure}[H]
\centering
\includegraphics[scale = 0.6]{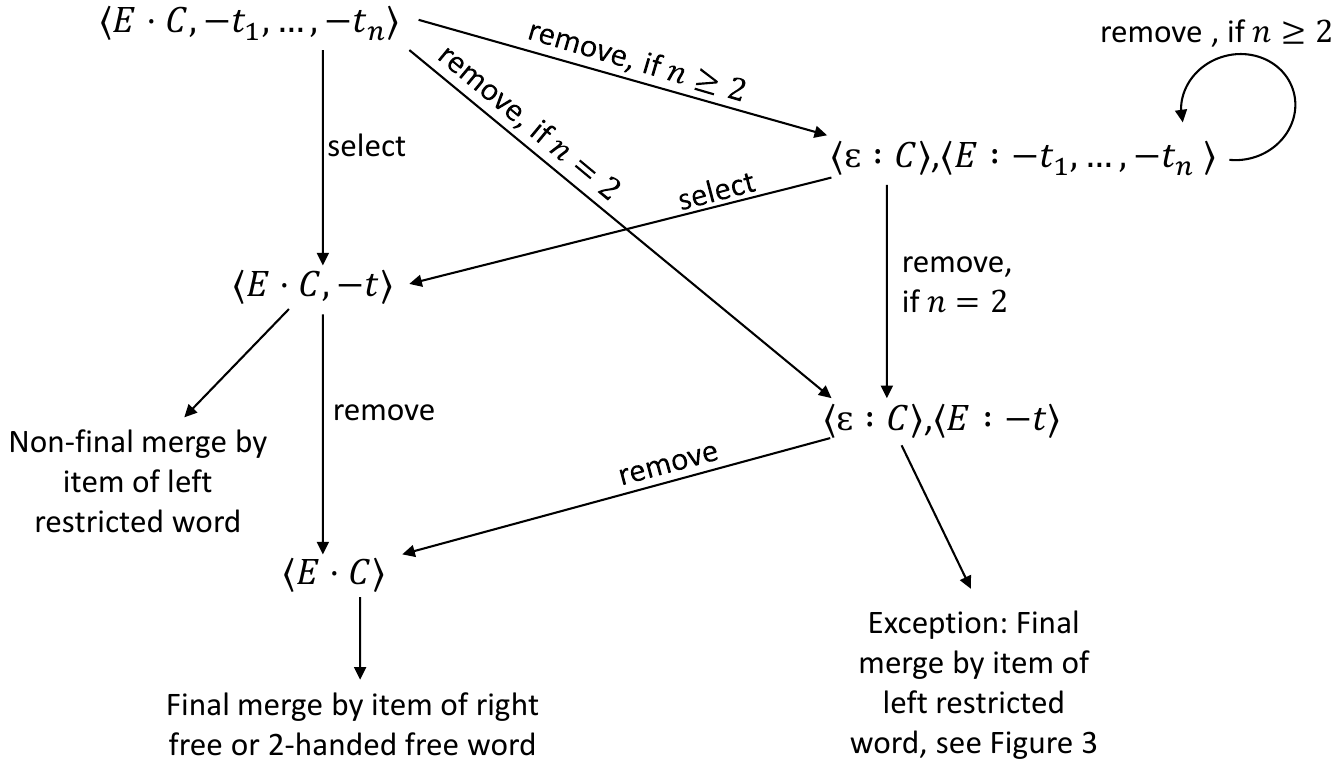}
\caption{The possible legal paths of derivation from an item of shape $\langle E\cdot \mathcal{C}, -t_1,\dots,-t_n\rangle$ or $\langle \epsilon :\mathcal{C}\rangle, \langle E:-t_1,\dots,-t_n\rangle$. Each arrow represents a merge- and potential consecutive move-operation(s) performed on the item that the arrow goes out of. The markers remove/select indicate that the arrow's merge- and move-operation(s) are triggered by a remove/select adapter.}
\label{pathsofadaptedderivation}
\end{figure}
Next, we take a deeper look into the possible shapes of word mergers, i.e. what does a $m_i^{[q_i]}$ do in the derivation path of Equation~\ref{derivoflexmergee}, if $m_i$ is (a self-derivation of) a word item. We show that
\begin{itemize}
    \item the operations triggered by $m_i$ always run analogously as in the scenarios described in Subsection~\ref{basic-wordshapes} or Figure~\ref{fig:vierundzwanzig}.
    \item the resulting product $m_i^{[q_i]}\circ\ldots\circ m_1^{[q_1]}(x)$ always has the shape $\langle\ldots\cdot \mathcal{C}_m,\ldots\rangle$.
\end{itemize}

Recall that (self-triggered derivations of) word items can only merge mergees of the shapes $\langle E\cdot \mathcal{C}\rangle$ or $\langle E\cdot \mathcal{C},-t\rangle$ or $\langle\epsilon :\mathcal{C}\rangle ,\langle E: -t\rangle$.
Consider a merger $m$ of category $\mathcal{C}_m$. It must hold ${=}\mathcal{C}$ or ${=}\mathcal{C},+t$ as its foremost selector-licensor block. $m$'s lexical head word $\overline{m}$ must also hold this block.

If the selector ${=}\mathcal{C}$ that triggers the merge with $x$ is $\overline{m}$'s foremost feature, then - assuming $m$ would be a non-trivial derivation of $\overline{m}$ - $m$ would not hold ${=}\mathcal{C}$, because this selector would have been triggered before. Then $m=\overline{m}$, so $m$ is a lexical word item. By Instruction 5 (or Observation~\ref{tab:MGitemtypes}) it is either right-free, 2-handed-free or left-restricted, so $m$ has the shape $\langle M::{=}\mathcal{C},\mathcal{C}_m,\ldots\rangle$ or $\langle M::{=}\mathcal{C},{=}\mathcal{C}_2,\mathcal{C}_m,\ldots\rangle$ or $\langle M::{=}\mathcal{C},+t,\mathcal{C}_m,\ldots\rangle$ respectively.

If ${=}\mathcal{C}$ is not $\overline{m}$'s foremost feature, then by Observation~\ref{tab:MGitemtypes} $\overline{m}$ must have 2 selector-licensor blocks and ${=}\mathcal{C}$ must be the second. Then $\overline{m}=\langle\overline{M}::{=}\mathcal{C}_1,{=}\mathcal{C},\mathcal{C}_m,\ldots\rangle$ and $m=\langle M:{=}\mathcal{C},\mathcal{C}_m,\ldots\rangle$.

The remaining possibilities of (merger,mergee)-pairs - when the merger is a (self-triggered derivation of a) word item - is listed in Table~\ref{tab:whenawordmerges} and each of the pairs lead to one the scenarios described in either Subsection~\ref{basic-wordshapes} or Figure~\ref{fig:vierundzwanzig}.

\begin{table}[ht]
    \centering
    \begin{tabular}{|l|l|m{2.3cm}|l|} \hline
        Merger & Mergee & Operation(s) & Resulting item \\ \hline
        $\langle M::{=}\mathcal{C},\mathcal{C}_m,\ldots\rangle$ & $\langle E\cdot \mathcal{C}\rangle$ & \makecell[l]{merge-1, \\see Eq.~\ref{merge:rf}} & $\langle ME:\mathcal{C}_m,\ldots\rangle$\\ \hline
        $\langle M::{=}\mathcal{C},{=}\mathcal{C}_2,\mathcal{C}_m,\ldots\rangle$ & $\langle E\cdot \mathcal{C}\rangle$ & \makecell[l]{merge-1, \\see Eq.~\ref{merge:2hf1}} & $\langle ME:{=}\mathcal{C}_2,\mathcal{C}_m,\ldots\rangle$ \\ \hline
        $\langle M:{=}\mathcal{C},\mathcal{C}_m,\ldots\rangle$ & $\langle E\cdot \mathcal{C}\rangle$ & \makecell[l]{merge-2, \\see Eq.~\ref{merge:2hf2}} & $\langle EM:\mathcal{C}_m,\ldots\rangle$\\ \hline
        $\langle M::{=}\mathcal{C},+t,\mathcal{C}_m,\ldots\rangle$ & $\langle E\cdot \mathcal{C},-t\rangle$ & merge-3 +move-1, see Eq.~\ref{merge:lr} + \ref{move:lr} & $\langle EM:\mathcal{C}_m,\ldots\rangle$\\ \hline
        $\langle M::{=}\mathcal{C},+t,\mathcal{C}_m,\ldots\rangle$ & $\langle\epsilon :\mathcal{C}\rangle ,\langle E: -t\rangle$) & merge-1 +move-1, see Figure~\ref{fig:vierundzwanzig} & $\langle EM:\mathcal{C}_m,\ldots\rangle$\\ \hline
    \end{tabular}
    \caption{All possible scenarios, how (a self-triggered derivation of) a word can merge in $G$. As can be seen, each scenario equals one of the scenarios described in Subsection~\ref{basic-wordshapes} or Figure~\ref{fig:vierundzwanzig}}
    \label{tab:whenawordmerges}
\end{table}

Unless the merger $m_i$ is $\langle M::{=}\mathcal{C},{=}\mathcal{C}_2,\mathcal{C}_m,\ldots\rangle$, it only has 1 selector and the outcome of merging a mergee $\hat{x}$ is always
\begin{align*}
    m_i^{[q_i]}(\hat{x})=m_i(\hat{x})=\langle \ldots:\mathcal{C}_m,\ldots\rangle.
\end{align*}

If the merger $m_i$ is $\langle M::{=}\mathcal{C},{=}\mathcal{C}_2,\mathcal{C}_m,\ldots\rangle$, it is the only case where it has 2 selectors and thus has a co-mergee $q_i$. The co-mergee $q_i$ must be merged by $m_i(\hat{x})=\langle ME:{=}\mathcal{C}_2,\mathcal{C}_m,\ldots\rangle$. So by Theorem~\ref{merge-criteria} $q_i$ must have the shape $\langle Q\cdot \mathcal{C}_2\rangle$. A merge-2 would be triggered, which would produce
\begin{align*}
    m_i^{[q_i]}(\hat{x})=m_i(\hat{x})(q_i)=\langle QME:\mathcal{C}_m,\ldots\rangle.
\end{align*}

In either case $m_i^{[q_i]}(\hat{x})$ has the shape $\langle\ldots:\mathcal{C}_m,\ldots\rangle$. Hence, all further derivations can again be fully described by the graph in Figure~\ref{pathsofadaptedderivation}, until a derivation has reached one of the mergee shapes of Table~\ref{tab:whenawordmerges}. Then the derivation can again be merged by (a self-derivation of) a word item, but again the following derivation would not deviate from the scenarios described in Subsection~\ref{basic-wordshapes} or Figure~\ref{fig:vierundzwanzig}.

Hence, there is no way for a derivation in a legal tree to escape the graph in Figure~\ref{pathsofadaptedderivation} and Table~\ref{tab:whenawordmerges}, so all merges and moves triggered by word items happen just as described in Subsection~\ref{basic-wordshapes}. So, in any leaf-to-root path of any derivation tree in $G$, each operation that affects the exponent is analogous to some derivation in $I$, where some NT is replaced by some word. No arrangement errors can occur, so the expression of the legal derivation tree in $G$ is also generatable by the CFG $I$. 
\end{proof}
\end{document}